\documentclass[journal]{IEEEtran}

\ifCLASSINFOpdf

\else

\fi

\usepackage{amsmath,amssymb,amsfonts}
\usepackage{graphicx}
\usepackage{comment}

\newtheorem{remark}{Remark}
\newtheorem{assumption}{Assumption}
\newtheorem{definition}{Definition}

\newtheorem{theorem}{Theorem}

\begin{document}

© 2026 IEEE. Personal use of this material is permitted.
Permission from IEEE must be obtained for all other uses,
including reprinting/republishing this material for advertising
or promotional purposes, collecting new collected works
for resale or redistribution to servers or lists, or reuse of
any copyrighted component of this work in other works.
This work has been submitted to the IEEE for possible
publication. Copyright may be transferred without notice,
after which this version may no longer be accessible.

\title{Flexibility-Aware Robotic Manipulation: \\ Deformation Estimation, Kinematics, and Control via Distributed Inertial Measurement}

\author{Amir Hossein~Barjini, Sadeq~Yaqubi, Pauli~Mustalahti, and~Jouni~Mattila     
\thanks{This work was supported by the Research Council of Finland under the Project "Nonlinear PDE-model-based control of flexible manipulators" under Grant 355664.}
\thanks{Authors are with the Department of Engineering and Natural Sciences, Tampere University, Finland.}
}

\maketitle

\begin{abstract}
Flexible robot links challenge the rigid-body assumption, making joint measurements alone insufficient for accurate manipulation. This paper introduces a flexibility-aware manipulation framework that connects structural dynamics directly to sensing, kinematics, and control. Instead of prescribing a dense inertial measurement unit (IMU) network, the dominant vibration modes of each flexible link are first identified from response-convergence and excitation-bandwidth criteria. Under a full-rank modal observability condition, $m$ retained modes require only $m$ independent slope measurements and one reference IMU, defining the minimum observable IMU configuration. The resulting measurements are used to reconstruct the continuous link deformation, which is then incorporated into flexibility-aware forward kinematics and a deformation-informed partial differential equation-based controller. Experiments on hydraulic single- and double-link manipulators evaluate the framework. On the long-reach double-link platform, the proposed kinematics achieves end-effector root mean square errors of $0.0107$~m, $0.0194$~m, and $0.0102$~rad while using fewer sensors than the piecewise-rigid method. Closed-loop reaching experiments further demonstrate real-time integration of IMU-estimated joint states and reconstructed deformation into the proposed controller. The results show that dominant modal information can determine sensing requirements and enable real-time flexibility-aware robotic manipulation.
\end{abstract}

\begingroup
\renewcommand{\abstractname}{Note to Practitioners}
\begin{abstract}
Long-reach manipulators used in drilling, inspection, construction, and material handling can bend enough that joint measurements no longer predict the tool pose accurately. External tracking systems are often impractical in factories or work sites because they require line of sight and additional infrastructure. The proposed method uses compact inertial sensors mounted along the robot links. A practitioner specifies the link dimensions, material properties, payload, and expected motion bandwidth; the procedure then determines how many deformation components must be represented and checks whether the selected sensor locations provide sufficient information. The resulting measurements can be used in real time to estimate link bending, improve end-effector pose estimation, and provide deformation information to the controller. Experiments on hydraulic manipulators indicate that this approach can reduce the required instrumentation relative to a dense piecewise-rigid layout while supporting accurate pose estimation and closed-loop operation. The current implementation assumes planar motion and small deformation, and the modal analysis should be updated when the payload or operating bandwidth changes substantially.
\end{abstract}
\endgroup


\IEEEpeerreviewmaketitle

\vspace{-10mm}
\section{Introduction}
\label{sec:introduction}

Robotic manipulators are increasingly required to perform dexterous and autonomous tasks in dynamic environments~\cite{billard2019trends}, where accurate state information underpins trajectory planning~\cite{ding2024trajectory}, task-space control~\cite{egli2022general,taheri2024end}, and controlled physical interaction~\cite{gold2022model,suomalainen2022survey}. 
However, as manipulator links become longer and payload demands increase, the conventional rigid-body assumption can become inadequate, since elastic deformation and vibration introduce discrepancies between rigid-model predictions and the actual robot motion, motivating the consideration of flexibility in heavy-duty manipulators~\cite{ding2025high,hubert2025vibration}. Advanced flexible robotic manipulation therefore requires not only joint measurements, but also real-time estimates of link deformation and their explicit incorporation into kinematics and control. The following subsection reviews the state of the art in these areas and identifies the remaining research gaps.

\vspace{-5mm}
\subsection{Related Works and Research Gap}
\label{subsec:literature_review}

Deformation estimation in flexible robots has been investigated through a range of sensing technologies and estimation approaches.
Optical measurement systems, including camera-based methods and laser trackers, provide noncontact measurements for pose estimation, shape reconstruction, and deformation assessment in rigid, soft, continuum, and flexible manipulators~\cite{lu2023image,barjini2025deep}.
Embedded optical-fiber sensors infer local curvature from distributed strain measurements, enabling three-dimensional shape sensing~\cite{paloschi20203d,sefati2019fbg}. 
In parallel, model-based state estimators fuse pose or strain measurements with continuum models to reconstruct the configuration of continuum robots~\cite{lilge2022continuum,anderson2017continuum}, while partial differential equation (PDE) observers estimate distributed flexible-link states from boundary measurements~\cite{yang2015observer,jiang2017robust}. Collectively, these approaches establish a broad foundation for deformation perception, with different requirements for external infrastructure, sensor integration, and model knowledge.

\begin{figure*}[htp]
    \centering
    \includegraphics[width=0.85\linewidth]{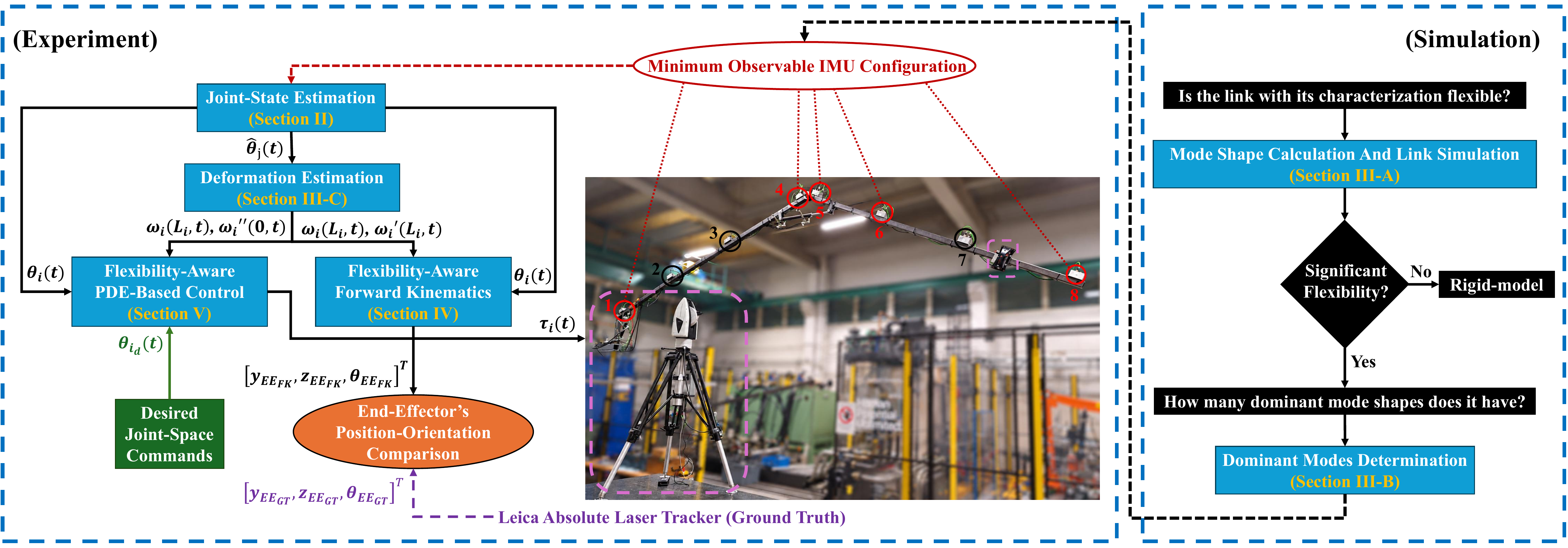}
    \caption{Overview of the proposed flexibility-aware manipulation framework, integrating mode selection, IMU-based deformation reconstruction, flexibility-aware kinematics, and control.}
    \label{fig:methodology}
\end{figure*}

Inertial measurement units (IMUs) offer a compact and cost-effective solution for robot perception, operating without external infrastructure, direct visual access, or specific lighting conditions, which makes them well suited to industrial applications~\cite{ahmad2013reviews}.
In rigid robotic manipulators, accelerometer and gyroscope measurements have enabled reliable joint-angle estimation~\cite{vihonen2017joint,cantelli2015joint}, while IMU-assisted methods have also supported collision monitoring and fault-tolerant state estimation~\cite{birjandi2020observer,baradaran2024fault}.
Observability-based studies have further established how the number and placement of inertial sensors influence kinematic-state estimation~\cite{fennel2023observability}.
Nevertheless, the use of distributed IMUs for continuous flexible-link deformation reconstruction, together with the relationship between the dominant modal content and the required sensor configuration, remains comparatively underexplored.

Flexible-link modeling consists of infinite-dimensional and finite-dimensional formulations. 
Infinite-dimensional PDE models preserve the continuous distributed dynamics and provide a high-fidelity description of flexible robot motion~\cite{mattioni2020infinite,yaqubi2026screw}. 
For real-time estimation and control, however, finite-dimensional representations are often preferred to reduce computational complexity~\cite{ding2025high}. Among these, finite-element methods discretize the link into elastic elements~\cite{piras2005dynamic,kermanian2019dynamic}, flexible-multibody formulations capture coupled rigid--flexible motion within complex mechanisms~\cite{scaglioni2017closed,khadim2025real}, and assumed-mode methods (AMMs) approximate deformation using a finite set of spatial basis functions~\cite{barjini2024deep,ding2025high}. 
The retained modal order is therefore crucial, requiring mode selection to balance model fidelity and computational efficiency while accounting for structural properties and payload variation~\cite{mishra2022determination,heidari2012truncation}. However, the relationship between the selected modes and the required sensor network remains largely unresolved.

Kinematics provides the bridge between deformation perception and task-space manipulation.
In continuum and soft robots, constant- and piecewise-constant-curvature models provide compact forward and inverse kinematic mappings~\cite{neppalli2009closed,webster2010design}, while geometrically exact rod formulations and geometric-constraint-based models have been used to describe robot shape and deformation under external loading~\cite{rucker2010geometrically,rucker2011statics,yang2020geometric}. Learning-based methods have also been investigated for inverse kinematics under highly nonlinear deformation~\cite{fang2022efficient}.
For flexible manipulators, measured or predicted deformation has been incorporated into kinematic transformations to correct end-effector position and orientation~\cite{gu2003flexible,kuccuk2004robot}. More recently, distributed-IMU approaches have represented flexible links as piecewise-rigid chains for real-time pose estimation~\cite{barjini2025ai} and subsystem-based control~\cite{tahamipour2025computationally}. These studies demonstrate the value of flexibility-aware kinematics; however, a sensor-efficient framework for serial flexible manipulators that reconstructs continuous link deformation using dominant mode shapes and explicitly incorporates it into the kinematics remains largely unexplored.

Accurate control of flexible manipulators requires accounting for elastic dynamics, since precise joint regulation alone cannot suppress link deformation and vibration or their resulting effect on the end-effector pose. Accordingly, existing strategies include visual servoing, neural-network and model-predictive control, reinforcement-learning-based control, and PDE-based control. 
A vision-based method combines nonlinear state observers with visual servoing to regulate end-effector motion while suppressing structural vibration~\cite{li2023nonlinear}. 
Neural-network and model-predictive controllers have also been experimentally investigated for motion regulation and vibration suppression~\cite{gao2018neural,silva2020implementable}. 
Actor--critic reinforcement learning has been employed for trajectory tracking and vibration suppression~\cite{pradhan2012real,he2020reinforcement,viswanadhapalli2024deep}. 
PDE-based controllers, in contrast, directly incorporate the distributed dynamics of flexible links into the control law~\cite{zhao2019boundary,yaqubi2023semi}, thereby avoiding modal truncation and enabling rigorous stability analysis. More recently, a nonlinear PDE controller has been coupled with a learning-based planner that generates vibration-minimizing reference trajectories~\cite{barjini2025deep}.
Despite these developments, deformation sensing and control have largely been treated as separate problems, although estimation latency and reconstruction errors can degrade closed-loop performance. In particular, the real-time integration of reconstructed link deformation into closed-loop control remains insufficiently investigated for heavy-duty serial flexible manipulators.

The literature reviewed above addresses individual elements of deformation sensing, flexibility-aware kinematics, and control, but does not provide the complete mode-informed sensing-to-control framework considered here. To the best of the authors' knowledge, no existing framework links dominant-mode selection to a minimum observable IMU configuration, uses the resulting measurements to reconstruct continuous link deformation, and incorporates the reconstructed deformation into both flexibility-aware kinematics and PDE-based closed-loop control. The overall workflow of the proposed framework, from flexible-link characterization and dominant-mode selection to reduced-order IMU-based deformation reconstruction, flexibility-aware kinematics, and closed-loop control, is summarized in Fig.~\ref{fig:methodology}.

\subsection{Contributions and Structure of the Paper}
\label{subsec:contribution}

The main contribution of this paper is a unified flexibility-aware robotic manipulation framework that integrates sensor-efficient deformation estimation, flexibility-aware kinematics, and control. The specific contributions are summarized as follows:
\begin{itemize}
\item A mode-selection procedure is developed to identify the dominant mode shapes of each flexible link based on response convergence and the system excitation bandwidth.

\item A sensor-efficient deformation estimation method is introduced, in which the retained modal coordinates are reconstructed using the identified dominant mode shapes. Under the modal observability condition, reconstructing $m$ dominant modes requires $m$ independent slope measurements and one reference IMU, defining the minimum observable IMU configuration.

\item The reconstructed link deformation is explicitly incorporated into flexibility-aware forward kinematics and a PDE-based joint-space controller, demonstrating its use in both end-effector pose estimation and real-time closed-loop control.

\item The proposed methods are validated on single-link flexible manipulator~(SLFM) and double-link flexible manipulator~(DLFM) using laser-tracker ground truth and are compared with conventional rigid and piecewise-rigid approaches.
\end{itemize}

The remainder of this paper is organized as follows. Section~\ref{sec:preliminaries} presents the inertial measurement methods for joint-state estimation. Section~\ref{sec:deformation} describes the flexible-link model, mode selection, and IMU-based deformation reconstruction. Sections~\ref{sec:kinematics} and~\ref{sec:control} develop the flexibility-aware kinematics and controller, respectively. Section~\ref{sec:results} reports the numerical and experimental results, and Section~\ref{sec:conclusion} concludes the paper.

\section{Inertial Measurement Preliminaries}
\label{sec:preliminaries}

The proposed framework relies on distributed IMUs to provide the joint-state information required for deformation reconstruction, flexibility-aware kinematics, and closed-loop control. This section briefly summarizes the estimation procedure adopted from~\cite{vihonen2017joint}, in which accelerometer and gyroscope measurements are combined to estimate the relative joint angle, while gyroscope measurements provide the relative joint angular velocity. The resulting signals are then combined through complementary filtering to reduce measurement noise and gyroscope drift.

\subsection{IMU-Based Joint-State Estimation}
\label{subsec:joint_state}

Consider the $(i-1)^\text{th}$ and $i^\text{th}$ links shown in Fig.~\ref{fig:links}, with joint frames $S_{i-1}$ and $S_i$ and corresponding IMU frames $P_{i-1}$ and $P_i$. The relative orientation between the adjacent links is described by the revolute-joint angle $\theta_i$.

\begin{figure}[h]
    \centering
    \includegraphics[width=0.65\linewidth]{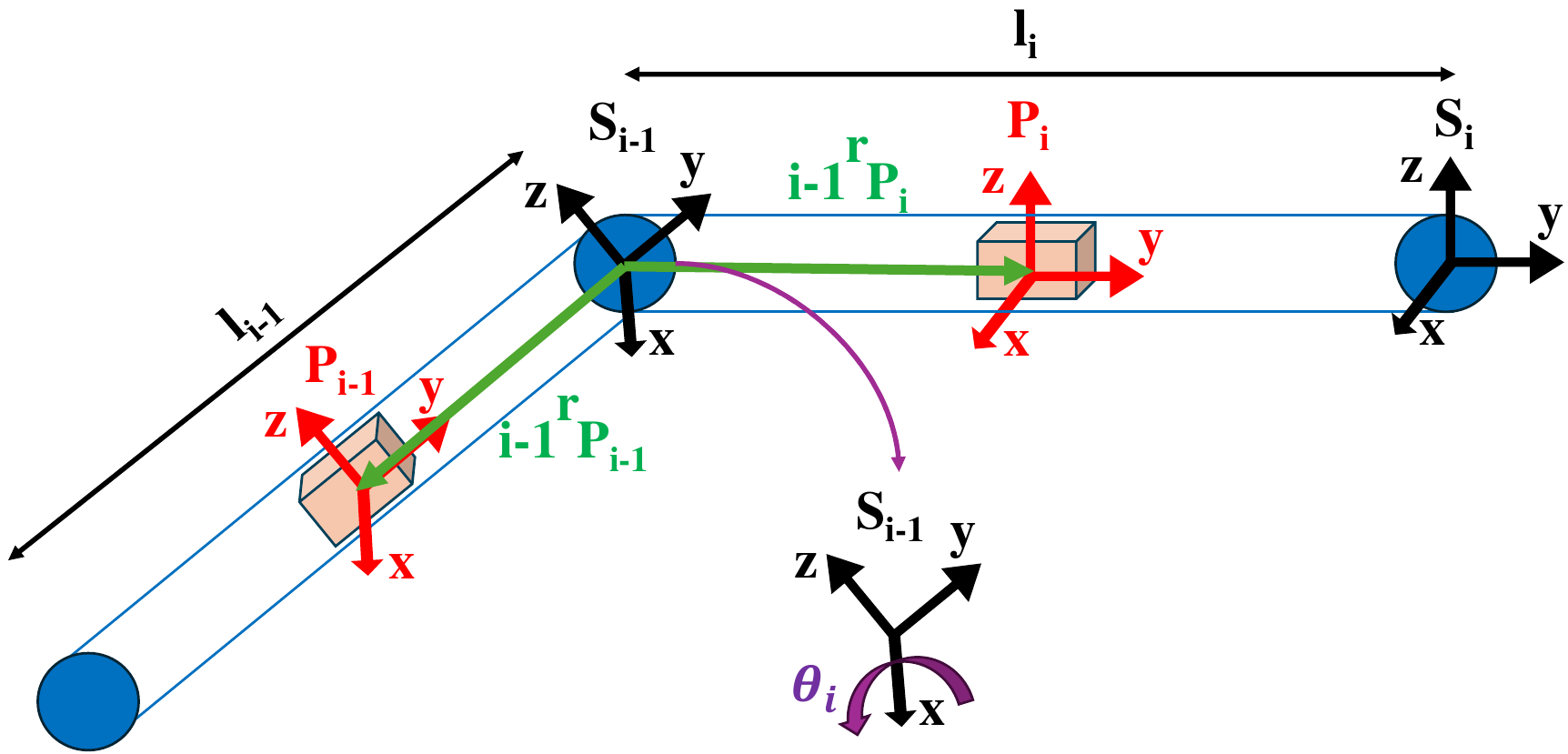}
    \caption{Adjacent-link frames and IMU configuration used for joint-state estimation.}
    \label{fig:links}
\end{figure}

For a revolute joint whose axis is aligned with the $x$-axis, the relative rotation between the two links is given by
\begin{equation}
{}^{i-1}\mathbf{R}_{i} =
\begin{bmatrix}
1 & 0 & 0 \\
0 & \cos\theta_i & -\sin\theta_i \\
0 & \sin\theta_i & \cos\theta_i
\end{bmatrix}.
\label{eq:Rotation_Matrix}
\end{equation}

Following the formulation in~\cite{vihonen2017joint}, the relative joint angle $\theta_i$ is estimated by combining the accelerometer measurements of adjacent IMUs with gyroscope information accounting for their relative motion
\begin{equation}
\theta_i =
\operatorname{atan2}\left(
\mathcal{A}_i,\mathcal{B}_i
\right),
\label{eq:theta}
\end{equation}
where
\begin{equation}
\begin{aligned}
\mathcal{A}_i
&=
{}^{i-1}\ddot{\mathbf{r}}^{z}_{P_{i-1}}
\left(
{}^{i}\ddot{\mathbf{r}}^{y}_{P_i}
-
{}^{i}_{P_{i-1}}\ddot{\mathbf{r}}^{y}_{P_i}
\right)
\\
&\quad-
{}^{i-1}\ddot{\mathbf{r}}^{y}_{P_{i-1}}
\left(
{}^{i}\ddot{\mathbf{r}}^{z}_{P_i}
-
{}^{i}_{P_{i-1}}\ddot{\mathbf{r}}^{z}_{P_i}
\right),
\\[1mm]
\mathcal{B}_i
&=
{}^{i-1}\ddot{\mathbf{r}}^{y}_{P_{i-1}}
\left(
{}^{i}\ddot{\mathbf{r}}^{y}_{P_i}
-
{}^{i}_{P_{i-1}}\ddot{\mathbf{r}}^{y}_{P_i}
\right)
\\
&\quad+
{}^{i-1}\ddot{\mathbf{r}}^{z}_{P_{i-1}}
\left(
{}^{i}\ddot{\mathbf{r}}^{z}_{P_i}
-
{}^{i}_{P_{i-1}}\ddot{\mathbf{r}}^{z}_{P_i}
\right).
\end{aligned}
\label{eq:theta_terms}
\end{equation}
Here, ${}^{i-1}\ddot{\mathbf{r}}_{P_{i-1}}$ and ${}^{i}\ddot{\mathbf{r}}_{P_i}$ are obtained from the corresponding accelerometer measurements, whereas ${}^{i}_{P_{i-1}}\ddot{\mathbf{r}}_{P_i}$ represents the relative acceleration between the two IMU locations and is determined from the gyroscope measurements and known sensor positions~\cite{vihonen2017joint}. Thus, the joint-angle estimate incorporates information from both accelerometers and gyroscopes.

The joint angular velocity is obtained directly from the relative angular velocity of the adjacent links. When expressed in frame $S_i$, this relative angular velocity is given by
\begin{equation}
{}^{i}_{i-1}\boldsymbol{\Omega}_{i}
=
{}^{i}\boldsymbol{\Omega}_{i}
-
{}^{i-1}\mathbf{R}_{i}^{T}
{}^{i-1}\boldsymbol{\Omega}_{i-1},
\label{eq:i-1_omega_i}
\end{equation}
where ${}^{i}\boldsymbol{\Omega}_{i}$ and
${}^{i-1}\boldsymbol{\Omega}_{i-1}$ are measured by the corresponding gyroscopes. Since the joint axis is aligned with the $x$-axis, the joint angular velocity follows as
\begin{equation}
\dot{\theta}_i
=
{}^{i}_{i-1}\boldsymbol{\Omega}_{i}^{x}.
\label{eq:theta_dot}
\end{equation}

\subsection{Complementary Filtering}
\label{subsec:complementary}

The IMU-based estimates of $\theta_i$ and $\dot{\theta}_i$ remain affected by measurement noise, sensor bias, and gyroscope drift. The gyroscope-based angular velocity $\dot{\theta}_i$ in~\eqref{eq:theta_dot} provides accurate short-term motion information but accumulates low-frequency drift when integrated. In contrast, the joint-angle estimate $\theta_i$ in~\eqref{eq:theta} provides a drift-bounded orientation reference but is more sensitive to high-frequency measurement noise.

To exploit these complementary characteristics, a complementary filter is employed to combine $\theta_i$ and $\dot{\theta}_i$, yielding the filtered joint-angle estimate $\hat{\theta}_i$, given by
\begin{equation}
\begin{bmatrix}
\dot{\hat{\theta}}_i \\
\dot{\hat{b}}_i
\end{bmatrix}
=
\begin{bmatrix}
0 & 1 \\
0 & 0
\end{bmatrix}
\begin{bmatrix}
\hat{\theta}_i \\
\hat{b}_i
\end{bmatrix}
+
\begin{bmatrix}
k_P \\
k_I
\end{bmatrix}
\left(
\theta_i-\hat{\theta}_i
\right)
+
\begin{bmatrix}
1 \\
0
\end{bmatrix}
\dot{\theta}_i,
\label{eq:filter}
\end{equation}
where $\hat{b}_i$ denotes the auxiliary integral state used to compensate for low-frequency bias and drift, while $k_P$ and $k_I$ are the proportional and integral filter gains, respectively.

\section{Deformation Estimation}
\label{sec:deformation}

Structural flexibility introduces a distributed state that is not captured by conventional joint measurements but directly influences the end-effector pose. Reliable flexibility-aware manipulation therefore requires online reconstruction of the deformation along each link. To this end, this section develops a sensor-efficient multi-IMU framework that determines the dominant modal content of each flexible link and reconstructs its continuous deformation using the minimum set of distributed orientation measurements satisfying the modal observability condition.

\subsection{Mode Shapes in Flexible Links}
\label{subsec:mode_shapes}

To account for link flexibility, each link is modeled as an Euler--Bernoulli beam, assuming a slender structure undergoing small transverse deformation with linear elastic behavior, while axial and shear deformations and rotary inertia are neglected~\cite{yang2015observer}. Under these assumptions, the transverse dynamics of the $i^\text{th}$ link are governed by
\begin{equation}
\rho_i A_i \ddot{\omega}_i(\xi_i, t) + E_i I_i \omega_i''''(\xi_i, t) = f(\xi_i, t),
\label{eq:Euler-Bernoulli}
\end{equation}
where $\omega_i(\xi_i, t)$ denotes the transverse deformation of the $i^\text{th}$ link at spatial coordinate $\xi_i$ and time $t$, as illustrated in Fig.~\ref{fig:deformation}. The parameters $\rho_i$, $A_i$, $E_i$, and $I_i$ represent the mass density, cross-sectional area, Young’s modulus, and second moment of area of the link, respectively. The term $f(\xi_i,t)$ denotes the distributed external load, while $L_i$ represents the link length.
\begin{figure}[ht]
    \centering
    \includegraphics[width=0.83\linewidth]{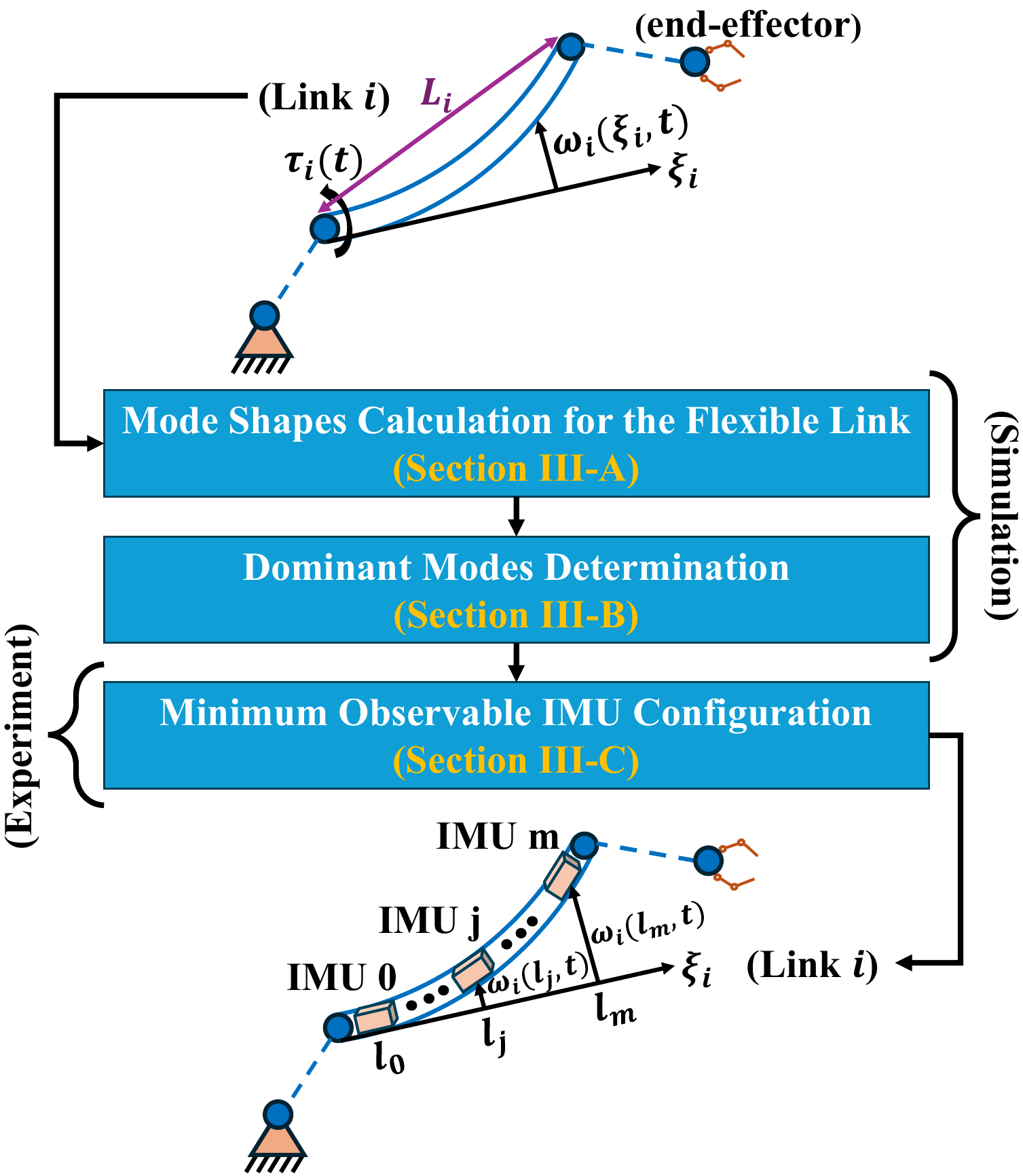}
    \caption{Mode-informed deformation reconstruction using the minimum observable IMU configuration for a flexible link.}
    \label{fig:deformation}
\end{figure}

Throughout this paper, a dot denotes partial differentiation with respect to time $\left( \dot{(\,\cdot\,)} = \partial(\cdot)/\partial t \right)$, 
and a prime denotes partial differentiation with respect to the spatial coordinate $\left( (\,\cdot\,)' = \partial(\cdot)/\partial \xi \right)$.

\begin{assumption}
Each flexible link is connected to the preceding link by a revolute joint and carries a payload mass at its tip. 
The total transverse displacement is decomposed into rigid-body rotation and elastic deformation.
As a result, the elastic deformation satisfies zero displacement and zero slope at the joint location, while the tip is free of externally applied bending moment and supports a concentrated mass.
\label{ass:boundary}
\end{assumption}

Considering Assumption~\ref{ass:boundary}, the boundary conditions for the flexible link are given by~\cite{yang2015observer}
\begin{equation}
\left\{
\begin{split}
    &\omega_i(0, t) = 0, \\
    &\omega_i'(0, t) = 0, \\
    &\omega_i''(L_i, t) = 0, \\
    &E_i I_i \, \omega_i'''(L_i, t) = M_{p_i} \, \ddot{\omega}_i(L_i, t),
\end{split}
\right.
\label{eq:boundary}
\end{equation}
where $M_{p_i}$ denotes the equivalent lumped mass at the free end of the $i^\text{th}$ link used for modal analysis, representing the combined contribution of the downstream components of the serial manipulator.

Using the AMM, the solution of~\eqref{eq:Euler-Bernoulli} can be expressed as
\begin{equation}
\omega_i(\xi_i, t) = \sum_{j=1}^{m} \Phi_{j_i}(\xi_i)\, \eta_{j_i}(t),
\label{eq:AMM}
\end{equation}
where $\Phi_{j_i}(\xi_i)$ denotes the $j^\text{th}$ mode shape and $\eta_{j_i}(t)$ is the corresponding time-dependent generalized coordinate of the $i^\text{th}$ link. Determining the number of mode shapes ($m$) required to accurately model the $i^\text{th}$ link involves performing a mode convergence analysis, which will be presented in the next section.

To calculate the mode shapes $\Phi_{j_i}(\xi_i)$, we consider the following general function~\cite{yaqubi2023semi}
\begin{equation}
\begin{split}
    \Phi_{j_i}(\xi_i) &= C_{1j}(\cos{\beta_j \xi_i}+\cosh{\beta_j \xi_i})\\
    &+ C_{2j}(\cos{\beta_j \xi_i}-\cosh{\beta_j \xi_i}) \\
    &+ C_{3j}(\sin{\beta_j \xi_i}+\sinh{\beta_j \xi_i}) \\
    &+ C_{4j}(\sin{\beta_j \xi_i}-\sinh{\beta_j \xi_i}).
\end{split}
\label{eq:phi}
\end{equation}  
Applying the boundary conditions in~\eqref{eq:boundary} yields
\begin{equation}
\begin{split}
    &C_{1j} = C_{3j} = 0,\\
    &\frac{C_{4j}}{C_{2j}} = - \left( \frac{ \cos{\beta_j L_i} + \cosh{\beta_j L_i} }{ \sin{\beta_j L_i} + \sinh{\beta_j L_i} } \right),
\end{split}
\label{eq:phi_coefficients}
\end{equation} 
resulting in the following characteristic equation
\begin{equation}
\begin{split}
    \frac{M_{p_i}}{\rho_i A_i} \beta_j \Big( \cos{\beta_j L_i} \sinh{\beta_j L_i} - \sin{\beta_j L_i} \cosh{\beta_j L_i} &\Big) \\
    + \cos{\beta_j L_i} \cosh{\beta_j L_i} + 1 &= 0,
\end{split}
\label{eq:characteristic}
\end{equation}
which can be solved numerically to obtain the values of $\beta_j$. Then the $j^\text{th}$ natural frequency of the $i^\text{th}$ link is calculated as
\begin{equation}
    f_{j_i} = \frac{1}{2 \pi} \beta_j^2 \sqrt{\frac{E_i I_i}{\rho_i A_i}}.
\label{eq:f}
\end{equation}

\subsection{Mode Selection Criteria for Flexible Links}
\label{subsec:mode_selection}

Following the derivation of the flexible-link mode shapes in Section~\ref{subsec:mode_shapes}, it is essential to determine the number of dominant modes. Identifying these modes is important for two reasons. First, in simulations, it enables accurate system modeling while limiting computational complexity. Second, from an experimental perspective, the number of IMUs required for each link is determined by the number of retained modes, as explained in Section~\ref{subsec:IMU_network}.

Two complementary criteria are used to determine the number of dominant modes of the $i^\text{th}$ link. The first ensures convergence of the reconstructed deformation, while the second ensures that all modes within the system excitation bandwidth are retained.

\subsubsection{Convergence of Observables Criterion}
\label{subsubsection:coo}

This criterion evaluates the relative change in the reconstructed tip deformation when an additional mode is included. For $k=1,2,\ldots$, it is defined as
\begin{equation}
C_{k_i}(t)
=
\frac{
\left|
\sum_{j=1}^{k+1}\Phi_{j_i}(L_i)\eta_{j_i}(t)
-
\sum_{j=1}^{k}\Phi_{j_i}(L_i)\eta_{j_i}(t)
\right|
}{
\left|
\sum_{j=1}^{k+1}\Phi_{j_i}(L_i)\eta_{j_i}(t)
\right|
+\varepsilon
},
\label{eq:coo}
\end{equation}
where $\delta$ is a predefined convergence threshold and $\varepsilon>0$ is a small regularization constant introduced to avoid numerical singularity when the reconstructed tip deformation approaches zero. Over the evaluation interval $\mathcal{T}$, the convergence-based modal order is selected as
\begin{equation}
m_i^{\mathrm{conv}}
=
\min
\left\{
k\in\mathbb{N}
\;\middle|\;
\max_{t\in\mathcal{T}}C_{k_i}(t)<\delta
\right\}.
\label{eq:coo_mode_selection}
\end{equation}

\subsubsection{Natural Frequency Bandwidth Criterion}
\label{subsubsection:nfb}

The second criterion retains the modes whose natural frequencies lie within the system excitation bandwidth, expressed as
\begin{equation}
f_{k_i}<\alpha f_{\mathrm{exc}},
\label{eq:nfb}
\end{equation}
where $f_{k_i}$ is the $k^\text{th}$ natural frequency of the $i^\text{th}$ link, computed from~\eqref{eq:f}, $\alpha$ is a safety factor, and $f_{\mathrm{exc}}$ denotes the maximum excitation frequency of the system. Since the natural frequencies are ordered increasingly, the bandwidth-based modal order is defined as
\begin{equation}
m_i^{\mathrm{bw}}
=
\max
\left\{
k\in\mathbb{N}
\;\middle|\;
f_{k_i}<\alpha f_{\mathrm{exc}}
\right\}.
\label{eq:nfb_mode_selection}
\end{equation}

The final number of dominant modes retained for the $i^\text{th}$ link is selected as
\begin{equation}
m_i
=
\max
\left\{
m_i^{\mathrm{conv}},
m_i^{\mathrm{bw}}
\right\}.
\label{eq:combined_mode_selection}
\end{equation}
Accordingly, the selected modal order satisfies both the response-convergence requirement and the excitation-bandwidth requirement.

\subsection{Minimum Observable IMU Configuration for Deformation Estimation in Flexible Links}
\label{subsec:IMU_network}

After determining the dominant mode shapes of each link through the convergence analysis in Section~\ref{subsec:mode_selection}, this section presents a deformation estimation methodology that utilizes these mode shapes in conjunction with a network of IMUs. The objective is to accurately reconstruct the elastic deformation of each flexible link while employing the minimum number of IMUs required to reconstruct the retained modal coordinates under the full-rank observability condition.

For generality, based on the convergence analysis presented in Section~\ref{subsec:mode_selection}, it is assumed that the first $m_i$ mode shapes adequately capture the dominant deformation dynamics of the $i^\text{th}$ link. Accordingly, the proposed methodology employs $(m_i+1)$ IMUs distributed along the link at known spatial locations, as illustrated in Fig.~\ref{fig:deformation}, where the reference IMU (IMU $0$) is located at $\xi_i = l_0 = 0$, at which $\omega'_i(0,t)=0$, and the $j^\text{th}$ IMU is located at $\xi_i = l_j$. Starting from the modal representation of the elastic deformation given in~\eqref{eq:AMM}, differentiation with respect to the spatial coordinate yields the local deformation slope:
\begin{equation}
\left\{
\begin{split}
\omega'_i(l_1, t)
&=
\Phi'_{1_i}(l_1)\,\eta_{1_i}(t)
+
\dots
+
\Phi'_{m_i}(l_1)\,\eta_{m_i}(t), \\
\omega'_i(l_2, t)
&=
\Phi'_{1_i}(l_2)\,\eta_{1_i}(t)
+
\dots
+
\Phi'_{m_i}(l_2)\,\eta_{m_i}(t), \\
&\vdots \\
\omega'_i(l_m, t)
&=
\Phi'_{1_i}(l_m)\,\eta_{1_i}(t)
+
\dots
+
\Phi'_{m_i}(l_m)\,\eta_{m_i}(t).
\end{split}
\right.
\label{eq:AMM_primes}
\end{equation}
The deformation slope at each IMU location is directly related to the measured relative orientation between IMUs. Specifically, the slope can be expressed as
\begin{equation}
\omega'_i(l_j, t)
=
\tan\left(\hat{\theta}_j(t)\right),
\end{equation}
where $\hat{\theta}_j(t)$ denotes the filtered relative angle between the reference IMU and the $j^\text{th}$ IMU, obtained using the complementary filter described in Section~\ref{sec:preliminaries} via~\eqref{eq:filter}. Furthermore, the spatial derivatives of the mode shapes, $\Phi'_{j_i}(\xi_i)$, are computed by differentiating the mode shape expression in~\eqref{eq:phi}, as presented in Section~\ref{subsec:mode_shapes}.
Therefore, the system of equations in~\eqref{eq:AMM_primes} can be compactly expressed in matrix form as follows
\begin{equation}
\mathbf{\Omega}'_i
=
\mathbf{\Phi}'_i
\boldsymbol{\eta}_i,
\label{eq:matrix_form}
\end{equation}
where
\begin{equation}
\mathbf{\Omega}'_i
=
\begin{bmatrix}
\omega'_i(l_1,t) \\
\omega'_i(l_2,t) \\
\vdots \\
\omega'_i(l_m,t)
\end{bmatrix}
=
\begin{bmatrix}
\tan(\hat{\theta}_1(t)) \\
\tan(\hat{\theta}_2(t)) \\
\vdots \\
\tan(\hat{\theta}_m(t))
\end{bmatrix}
\in \mathbb{R}^{m_i},
\end{equation}
\begin{equation}
\mathbf{\Phi}'_i
=
\begin{bmatrix}
\Phi'_{1_i}(l_1) & \Phi'_{2_i}(l_1) & \dots & \Phi'_{m_i}(l_1) \\
\Phi'_{1_i}(l_2) & \Phi'_{2_i}(l_2) & \dots & \Phi'_{m_i}(l_2) \\
\vdots & \vdots & \ddots & \vdots \\
\Phi'_{1_i}(l_m) & \Phi'_{2_i}(l_m) & \dots & \Phi'_{m_i}(l_m)
\end{bmatrix}
\in \mathbb{R}^{m_i \times m_i},
\end{equation}
\begin{equation}
\boldsymbol{\eta}_i
=
\begin{bmatrix}
\eta_{1_i}(t) \\
\eta_{2_i}(t) \\
\vdots \\
\eta_{m_i}(t)
\end{bmatrix}
\in \mathbb{R}^{m_i}.
\end{equation}
To ensure unique reconstruction of the retained modal coordinates, the IMU locations must be chosen such that the modal slope matrix satisfies the observability condition
\begin{equation}
\operatorname{rank}\left(\mathbf{\Phi}'_i\right)=m_i.
\label{eq:modal_observability}
\end{equation}
Under this condition, since both $\mathbf{\Omega}'_i$ and $\mathbf{\Phi}'_i$ are known, the modal coordinates $\boldsymbol{\eta}_i$ can be uniquely computed from~\eqref{eq:matrix_form}. Accordingly, reconstructing $m_i$ dominant modes requires $m_i$ independent slope measurements and one reference IMU. 

Consequently, the full elastic deformation of the link, $\omega_i(\xi_i,t)$, can be reconstructed by substituting the estimated modal coordinates $\boldsymbol{\eta}_i(t)$ into the modal expansion~\eqref{eq:AMM}.

\section{Flexibility-Aware Kinematics}
\label{sec:kinematics}
Consider a general planar serial manipulator with $n$ links, as illustrated in Fig.~\ref{fig:FK}, where link flexibility is also taken into account. Since deformation estimation was addressed in Section~\ref{sec:deformation}, we can now reconstruct the forward kinematics algorithms, defined as flexibility-aware forward kinematics.

As shown in Fig.~\ref{fig:FK}, frame $\{W\}$ denotes the inertial frame. Due to link flexibility, assigning a single body-fixed coordinate system to each link is insufficient. Therefore, two body-fixed coordinate systems are attached to each link, one at the base and one at the tip, denoted by $\{B_i\}$ and $\{T_i\}$, respectively.
\begin{figure}[h]
    \centering
    \includegraphics[width=0.78\linewidth]{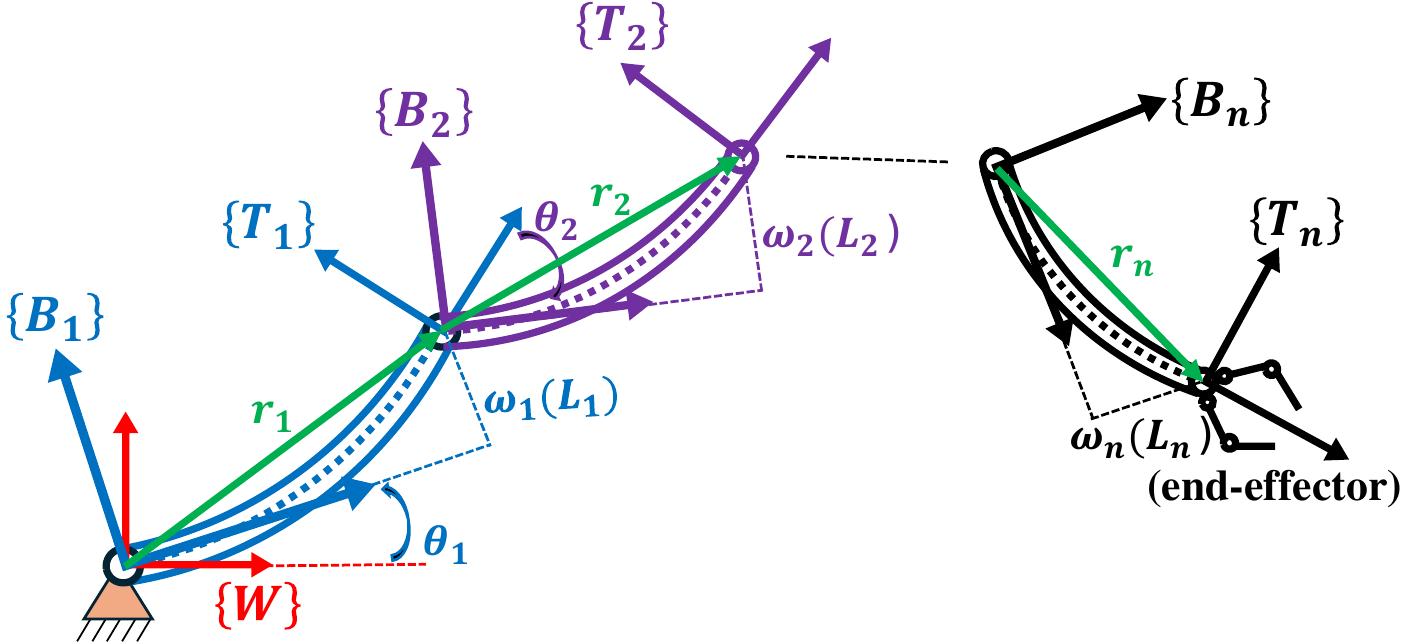}
    \caption{Frame assignment for flexibility-aware forward kinematics of a planar serial flexible manipulator.}
    \label{fig:FK}
\end{figure}

\begin{definition}[Euclidean and Configuration Spaces]
\label{def:R_G}
We denote by $\mathbb{R}^n$ the $n$-dimensional Euclidean space. The configuration space of planar rigid-body motions is described by the special orthogonal group $SO(2)$ and the special Euclidean group $SE(2)$, defined as follows
\begin{equation}
    SO(2) := \left\{ R \in \mathbb{R}^{2\times 2} \mid R^\top R = I,\ \det(R) = 1 \right\},
\end{equation}
\begin{equation}
    SE(2) := \left\{
        G =
        \begin{bmatrix}
            R & p\\
            0 & 1
        \end{bmatrix}
        \,\middle|\,
        R \in SO(2),\ p \in \mathbb{R}^2
    \right\}.
\end{equation}
\end{definition}

Now, consider the $i^{\text{th}}$ link, which is connected to the $(i-1)^{\text{th}}$ link via a revolute joint, as illustrated in Fig.~\ref{fig:FK}. The homogeneous transformation matrices required for the forward kinematics of the robotic manipulator are given by
\begin{equation}
\begin{split}
    {}^{T_{i-1}}G_{B_i} &=
        \begin{bmatrix}
            {}^{T_{i-1}}R_{B_i} & 0_{2\times1}\\
            0_{1\times2} & 1
        \end{bmatrix} \in SE(2), \\
    {}^{T_{i-1}}R_{B_i} &=
        \begin{bmatrix}
            \cos{\theta_i} & -\sin{\theta_i}\\
            \sin{\theta_i} & \cos{\theta_i}
        \end{bmatrix} \in SO(2),
\end{split}
\label{Ti1_G_Bi}
\end{equation}
\begin{equation}
\begin{split}
    {}^{B_i}G_{T_i} &=
        \begin{bmatrix}
            {}^{B_i}R_{T_i} & {}^{B_i}r_i\\
            0_{1\times2} & 1
        \end{bmatrix} \in SE(2), \\
    {}^{B_i}R_{T_i} &=
        \begin{bmatrix}
            \cos{\theta_{end_i}} & -\sin{\theta_{end_i}}\\
            \sin{\theta_{end_i}} & \cos{\theta_{end_i}}
        \end{bmatrix} \in SO(2),
\end{split}
\label{Bi_G_Ti}
\end{equation}
where $\theta_i(t)$ denotes the relative joint angle, as defined in Section~\ref{sec:preliminaries}, and $\theta_{\mathrm{end}_i}(t)$ represents the orientation of the link tip resulting from elastic deformation, as derived in Section~\ref{sec:deformation}
\begin{equation}
    \theta_{\mathrm{end}_i}(t) = \tan^{-1}\!\big(\omega'_i(L_i, t)\big).
\label{eq:theta_end_i}
\end{equation}
It is worth noting that $\{T_0\} = \{W\}$.

\begin{assumption}
\label{ass}
Each flexible link is assumed to undergo small transverse deformation, such that
\begin{equation}
\frac{|\omega_i(L_i,t)|}{L_i}\ll 1,
\qquad
\max_{\xi_i\in[0,L_i]}
|\omega'_i(\xi_i,t)|\ll 1.
\label{eq:small_deformation_condition}
\end{equation}
Under this condition, axial deformation is neglected, and the deformed link is approximated by the straight chord connecting its base and tip. Therefore, the position of frame $\{T_i\}$ relative to frame $\{B_i\}$, expressed in frame $\{B_i\}$, is assumed as
\begin{equation}
    {}^{B_i}r_i(t) =
    \begin{bmatrix}
        \sqrt{L_i^2 - \omega_i^2(L_i, t)} \\
        \omega_i(L_i)
    \end{bmatrix}.
    \label{r_i}
\end{equation}
\end{assumption}

The forward kinematics of the robotic manipulator is therefore obtained as
\begin{equation}
    {}^{W}G_{T_n} = \prod_{i=1}^{n} \left( {}^{T_{i-1}}G_{B_i} \, {}^{B_i}G_{T_i} \right)
    =
    \begin{bmatrix}
        {}^{W}R_{T_n} & {}^{W}p_{T_n}\\
        0_{1\times2} & 1
    \end{bmatrix}.
\end{equation}
In this expression, ${}^{W}R_{T_n}\in SO(2)$ represents the orientation of the end-effector frame, parameterized by the end-effector angle $\theta_{EE}$, while ${}^{W}p_{T_n}=[y_{EE},z_{EE}]^T\in\mathbb{R}^2$ denotes its position. Both quantities are defined relative to and expressed in the inertial frame ${W}$.

Hence, the flexibility-aware kinematic model maps the IMU-estimated joint states and reconstructed link deformation into the corresponding end-effector position and orientation. The same reconstructed flexible-link state is subsequently used in the control formulation, where deformation-dependent quantities are incorporated into the closed-loop control.

\section{Flexibility-Aware Controller Design}
\label{sec:control}
The previous section used the reconstructed link deformation in the geometric mapping from joint space to the end-effector pose. The same deformation information is now incorporated at the dynamic level for closed-loop control. Specifically, while the flexibility-aware kinematics uses the reconstructed tip deformation and tip slope to determine the deformed link pose, the controller uses deformation-dependent quantities, including the link-tip deformation and base curvature, within the PDE-based control law. Thus, the kinematic and control formulations constitute two complementary uses of the same IMU-reconstructed flexible state within the proposed sensing-to-control framework.

%As discussed previously, the deformation of a link, computed using the method introduced in Section~\ref{sec:deformation}, was incorporated in Section~\ref{sec:kinematics} to develop flexibility-aware forward kinematics. In this section, we investigate an additional application of the computed deformation, namely its integration into closed-loop control algorithms.

As the equations of motion for flexible links are governed by PDEs, incorporating these PDEs into the controller design enables a more accurate representation of the system dynamics within the control law. Therefore, prior to designing the controller, the equations of motion of the flexible manipulator must be derived. Inspired by the modeling and controller design procedure for SLFMs presented in~\cite{barjini2025deep}, this work extends the approach to DLFMs, as illustrated in Fig.~\ref{fig:control}.

\begin{figure}[h]
    \centering
    \includegraphics[width=0.74\linewidth]{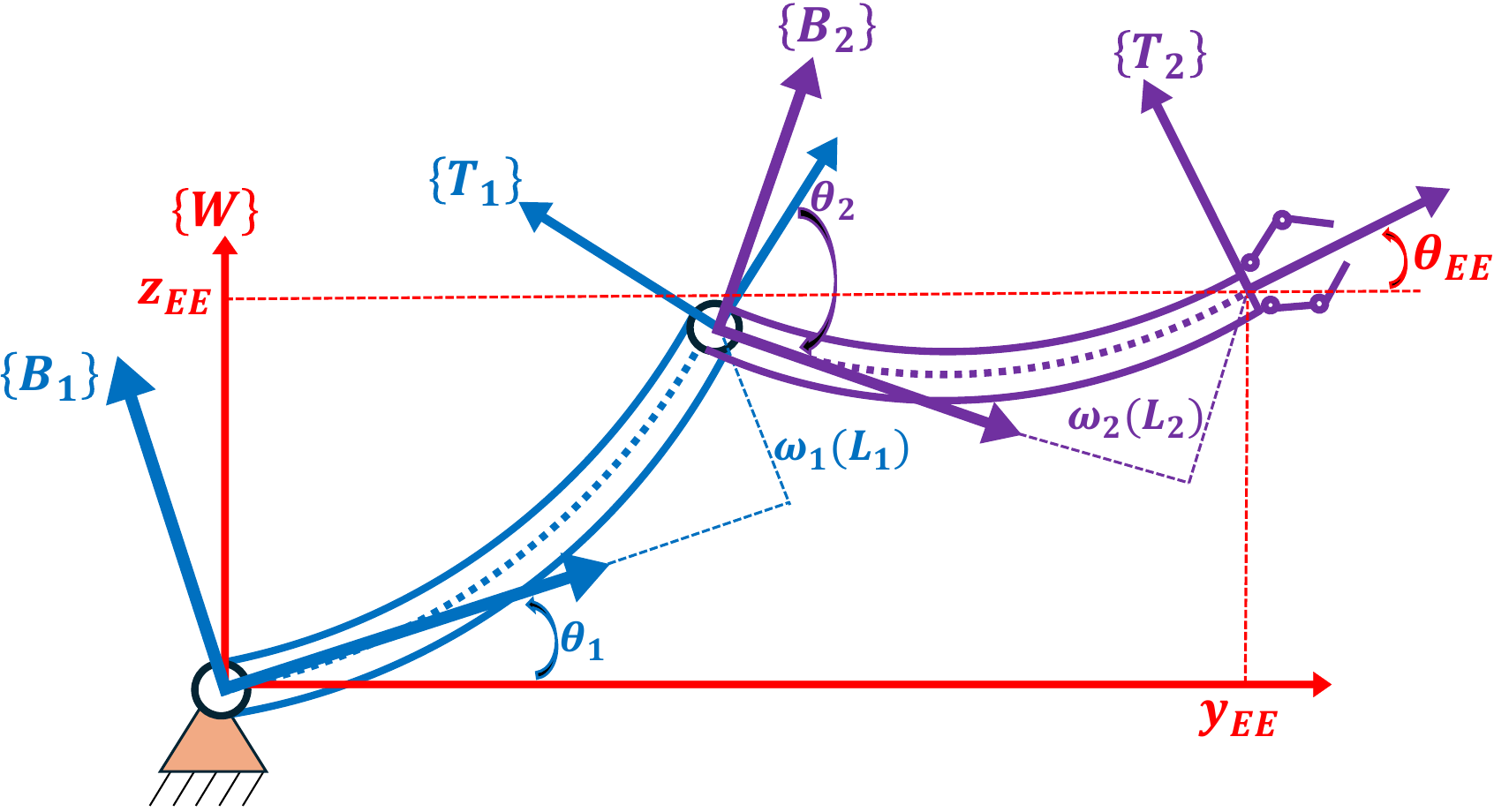}
    \caption{DLFM configuration and deformation variables used for the deformation-informed PDE-based controller.}
    \label{fig:control}
\end{figure}

\subsection{Dynamics Modeling of DLFM}
\label{subsec:mathematical_modeling}
For the controller design, the DLFM is described in the planar domain and each flexible link is modeled as an Euler--Bernoulli beam. The formulation assumes small transverse deformation, while axial and shear deformations are neglected. The total link motion is therefore represented by the superposition of rigid-body joint motion and elastic deformation, with the corresponding boundary conditions given in Assumption~\ref{ass:boundary}.

For the DLFM illustrated in Fig.~\ref{fig:control}, the total transverse displacements of the first and second flexible links are therefore expressed as
\begin{equation}
    \begin{split}
        s_1(\xi_1, t) &= \xi_1 \theta_1(t) + \omega_1(\xi_1, t), \\
        s_2(\xi_2, t) &= L_1 \theta_1(t) + \omega_1(L_1, t) \\
        &\quad + \xi_2 \left[\theta_1(t)+\theta_2(t)\right]
        + \omega_2(\xi_2, t).
    \end{split}
\label{eq:s_i}
\end{equation}
In the next step, to derive the equations of motion, the kinetic energy ($T$), potential energy ($U$), and virtual work done by external torques ($\delta W$), can be expressed as
\begin{equation}
    \begin{split}
        T &= \frac{1}{2} I_{h_1} \dot \theta_1^2(t) + \frac{1}{2} \rho_1 A_1 \int_0^{L_1} \dot s_1^2(\xi_1, t)  d\xi_1 + \frac{1}{2} M_1 \dot s_1^2(L_1, t) \\
          &+ \frac{1}{2} I_{h_2} \dot \theta_2^2(t) + \frac{1}{2} \rho_2 A_2 \int_0^{L_2} \dot s_2^2(\xi_2, t)  d\xi_2 + \frac{1}{2} M_2 \dot s_2^2(L_2, t),
    \end{split} 
\label{eq:T}
\end{equation}
\begin{equation}
    \begin{split}
        U &= m_1 g \frac{L_1}{2}  \sin{\theta_1(t)} \\
          &+ M_1 g [L_1 \sin{\theta_1(t)} + \omega_1(L_1, t)\cos{\theta_1(t)}]\\
          &+ m_2 g [L_1 \sin{\theta_1(t)} + \frac{L_2}{2}  \sin{ ( \theta_1(t)  + \theta_2(t)  )} ] \\
          &+ M_2 g [L_1 \sin{\theta_1(t)} + \omega_1(L_1, t)\cos{\theta_1(t)} \\
          &\quad \, \, +   L_2 \sin{ ( \theta_1(t) + \theta_2(t) )}  + \omega_2(L_2, t)\cos{ ( \theta_1(t) + \theta_2(t) )}  ] \\
          &+ \frac{1}{2} E_1 I_1 \int_0^{L_1} [s_1''(\xi_1, t)]^2 d\xi_1 + \frac{1}{2} E_2 I_2 \int_0^{L_2} [s_2''(\xi_2, t)]^2 d\xi_2,
    \end{split}
\label{eq:U}
\end{equation}
\begin{equation}
\delta W
=
\tau_1(t)\,\delta\theta_1(t)
+
\tau_2(t)\,\delta\theta_2(t).
\label{eq:virtual_work}
\end{equation}
where \( \tau_i \), \( I_{h_i} \), \( \rho_i \), \( A_i \), \( m_i \), \( L_i \), \( E_i I_i \), and \( M_i \) denote the torque, hub rotational inertia, density, cross-sectional area, link mass, link length, bending stiffness, and payload of the $i^\text{th}$ flexible link, respectively.

Now, by applying the extended Hamilton’s principle, $\int_{t_1}^{t_2} \left( \delta T - \delta U + \delta W \right) dt = 0$,
the governing equations of motion can be derived as follows
\begin{equation}
    \begin{split}
        &I_{h_1} \ddot{\theta}_1(t) - E_1 I_1 \omega_1''(0, t) - E_2 I_2 \omega_2''(0, t) \\
        &+ \left(\frac{1}{2} m_1 + m_2 \right) g L_1 \cos(\theta_1(t)) \\
        &- (M_1 + M_2) g \omega_1(L_1, t) \sin(\theta_1(t)) \\
        &+ \frac{1}{2} m_2 g L_2 \cos(\theta_1(t) + \theta_2(t)) \\
        &- M_2 g \omega_2(L_2, t) \sin(\theta_1(t) + \theta_2(t)) = \tau_1(t),
    \end{split}
\label{eq:thetaddot_1}
\end{equation}
\begin{equation}
    \begin{split}
        &I_{h_2} \ddot{\theta}_2(t) - E_2 I_2 \omega_2''(0, t) \\
        &+ \frac{1}{2} m_2 g L_2 \cos(\theta_1(t) + \theta_2(t)) \\
        &- M_2 g \omega_2(L_2, t) \sin(\theta_1(t) + \theta_2(t)) = \tau_2(t),
    \end{split}
\label{eq:thetaddot_2}
\end{equation}
\begin{equation}
    \begin{split}
        \rho_1 A_1 \left( \xi_1 \ddot{\theta}_1(t) + \ddot{\omega}_1(\xi_1, t) \right) + E_1 I_1 \omega''''_1(\xi_1, t) = 0,
    \end{split}
\label{eq:omegaddot_1}
\end{equation}
\begin{equation}
    \begin{split}
        &\rho_2 A_2 \left( L_1 \ddot{\theta}_1(t) + \ddot{\omega}_1(L_1, t) \right. \\   
        &\left. \quad \, \, +\xi_2 [\ddot{\theta}_1(t) + \ddot{\theta}_2(t)]  + \ddot{\omega}_2(\xi_2, t)  \right) + E_2 I_2 \omega''''_2(\xi_2, t) = 0,
    \end{split}
\label{eq:omegaddot_2}
\end{equation}
along with the associated boundary conditions, as
\begin{equation}
    \begin{split}
        &M_1 \left( L_1 \ddot{\theta}_1(t) + \ddot{\omega}_1(L_1, t) \right) \\
        &+ (M_1 + M_2) g \cos(\theta_1(t)) - M_2 g \cos(\theta_1(t) + \theta_2(t)) \\
        &= E_1 I_1 \omega'''_1(L_1, t) ,
    \end{split}
\label{eq:BC1}
\end{equation}
\begin{equation}
    \begin{split}
        &M_2 \left( L_1 \ddot{\theta}_1(t) + \ddot{\omega}_1(L_1, t)  + L_2 [\ddot{\theta}_1(t) + \ddot{\theta}_2(t)]  + \ddot{\omega}_2(L_2, t) \right) \\
        &+ M_2 g \cos(\theta_1(t) + \theta_2(t)) = E_2 I_2 \omega'''_2(L_2, t) ,
    \end{split}
\label{eq:BC2}
\end{equation}
\begin{equation}
    \omega_i(0,t) = 0, \, \omega_i'(0,t) = 0, \, \omega_i''(L_i,t) = 0, \quad i = 1,2.
    \label{eq:BC3}
\end{equation}

\subsection{PDE-Based Controller Design and Stability Analysis}
\label{subsec:pde_controller}
As shown in Section~\ref{subsec:mathematical_modeling}, the dynamics of flexible structures are inherently described by PDEs. Embedding these dynamics directly into the controller design framework enables a more precise characterization of the system behavior. Accordingly, the deformation-informed PDE-based control law for the DLFM, illustrated in Fig.~\ref{fig:control}, is defined as
\begin{equation}
\left\{
\begin{split}
    \tau_i(t) &= K_{p_i} e_i(t) + K_{d_i} \dot{e}_i(t) - H_i(t), \quad i = 1,2 \\
    H_1(t) &=  E_1 I_1 \omega_1''(0, t) + E_2 I_2 \omega_2''(0, t) \\
    &- \left(\frac{1}{2} m_1 + m_2 \right) g L_1 \cos(\theta_1(t)) \\
    &+ (M_1 + M_2) g \omega_1(L_1, t) \sin(\theta_1(t)) \\
    &- \frac{1}{2} m_2 g L_2 \cos(\theta_1(t) + \theta_2(t)) \\
    &+ M_2 g \omega_2(L_2, t) \sin(\theta_1(t) + \theta_2(t)) - I_{h_1} \ddot{\theta}_{d_1}, \\
    H_2(t) &=  E_2 I_2 \omega_2''(0, t) - \frac{1}{2} m_2 g L_2 \cos(\theta_1(t) + \theta_2(t)) \\
    &+ M_2 g \omega_2(L_2, t) \sin(\theta_1(t) + \theta_2(t)) - I_{h_2} \ddot{\theta}_{d_2},
\end{split}
\right.
\label{eq:tau_i}
\end{equation}
where $e_i(t) = \theta_{d_i}(t) - \theta_i(t)$ denotes the tracking error, and $K_{p_i}$ and $K_{d_i}$ are positive control gains. 

It is important to emphasize that, in the control law~\eqref{eq:tau_i}, both the tip deformation of each link, $\omega_i(L_i, t)$, and the second spatial derivative of the deformation at the base of each link, $\omega_i''(0, t)$, are obtained using the proposed deformation estimation method described in Section~\ref{sec:deformation}.

\begin{theorem}
\label{thm:tracking_stability}
Under exact compensation of the modeled deformation-dependent, gravity, and feedforward terms, the control law in~\eqref{eq:tau_i} guarantees exponential stability of the joint-angle tracking-error subsystem, provided that $I_{h_i}>0$, $K_{p_i}>0$, and $K_{d_i}>0$ for $i=1,2$. Consequently, the tracking errors $e_i(t)$ and $\dot e_i(t)$ converge exponentially to zero.
\end{theorem}

\textit{Proof:}
Substituting the control law~\eqref{eq:tau_i} into the joint dynamics in~\eqref{eq:thetaddot_1} and~\eqref{eq:thetaddot_2}, and assuming exact compensation of the modeled deformation-dependent, gravity, and feedforward terms, yields the closed-loop tracking-error dynamics
\begin{equation}
I_{h_i}\ddot e_i(t)
+
K_{d_i}\dot e_i(t)
+
K_{p_i}e_i(t)
=
0,
\qquad i=1,2.
\label{eq:closed_loop_error}
\end{equation}
For each joint, define the tracking-error state vector as
\begin{equation}
\mathbf{x}_i(t)
=
\begin{bmatrix}
e_i(t) &
\dot e_i(t)
\end{bmatrix}^{\top}.
\label{eq:error_state}
\end{equation}
Using this definition, the error dynamics in~\eqref{eq:closed_loop_error} can be expressed in the state-space form
\begin{equation}
\dot{\mathbf{x}}_i(t)
=
\mathbf{A}_i\mathbf{x}_i(t),
\qquad
\mathbf{A}_i
=
\begin{bmatrix}
0 & 1\\[1mm]
-\dfrac{K_{p_i}}{I_{h_i}} &
-\dfrac{K_{d_i}}{I_{h_i}}
\end{bmatrix},
\label{eq:error_state_matrix}
\end{equation}
where $\mathbf{A}_i$ denotes the closed-loop state matrix of the $i^\text{th}$ joint tracking-error subsystem.
The characteristic polynomial associated with
$\mathbf{A}_i$ is
\begin{equation}
s^2
+
\frac{K_{d_i}}{I_{h_i}}s
+
\frac{K_{p_i}}{I_{h_i}}.
\label{eq:error_characteristic_polynomial}
\end{equation}
Since $I_{h_i}>0$, $K_{p_i}>0$, and $K_{d_i}>0$, both coefficients of the characteristic polynomial are positive. Therefore, $\mathbf{A}_i$ is Hurwitz, i.e., all of its eigenvalues have strictly negative real parts.
For a Hurwitz matrix $\mathbf{A}_i$, and for any selected symmetric positive-definite matrix $\mathbf{Q}_i=\mathbf{Q}_i^{\top}\succ0$, there exists a unique symmetric positive-definite matrix $\mathbf{P}_i=\mathbf{P}_i^{\top}\succ0$ satisfying the continuous-time Lyapunov equation
\begin{equation}
\mathbf{A}_i^{\top}\mathbf{P}_i
+
\mathbf{P}_i\mathbf{A}_i
=
-\mathbf{Q}_i.
\label{eq:lyapunov_matrix_equation}
\end{equation}
Consider the quadratic Lyapunov function
\begin{equation}
V(\mathbf{x})
=
\sum_{i=1}^{2}
\mathbf{x}_i^{\top}\mathbf{P}_i\mathbf{x}_i.
\label{eq:quadratic_lyapunov}
\end{equation}
Since $\mathbf{P}_i\succ0$, $V(\mathbf{x})$ is positive definite and satisfies $V(\mathbf{x})>0$ for every $\mathbf{x}\neq\mathbf{0}$, with $V(\mathbf{0})=0$.
Taking the time derivative of~\eqref{eq:quadratic_lyapunov} along~\eqref{eq:error_state_matrix} gives
\begin{equation}
\begin{split}
\dot V
&=
\sum_{i=1}^{2}
\left(
\dot{\mathbf{x}}_i^{\top}\mathbf{P}_i\mathbf{x}_i
+
\mathbf{x}_i^{\top}\mathbf{P}_i\dot{\mathbf{x}}_i
\right)
\\
&=
\sum_{i=1}^{2}
\mathbf{x}_i^{\top}
\left(
\mathbf{A}_i^{\top}\mathbf{P}_i
+
\mathbf{P}_i\mathbf{A}_i
\right)
\mathbf{x}_i
\\
&=
-
\sum_{i=1}^{2}
\mathbf{x}_i^{\top}\mathbf{Q}_i\mathbf{x}_i.
\end{split}
\label{eq:quadratic_lyapunov_derivative}
\end{equation}
To express the result compactly, define the overall tracking-error state as
\begin{equation}
\mathbf{x}
=
\begin{bmatrix}
\mathbf{x}_1^{\top} &
\mathbf{x}_2^{\top}
\end{bmatrix}^{\top},
\end{equation}
and the corresponding block-diagonal matrices as
\begin{equation}
\mathbf{P}
=
\operatorname{blkdiag}
\left(
\mathbf{P}_1,\mathbf{P}_2
\right),
\qquad
\mathbf{Q}
=
\operatorname{blkdiag}
\left(
\mathbf{Q}_1,\mathbf{Q}_2
\right).
\end{equation}
Accordingly,
\begin{equation}
V
=
\mathbf{x}^{\top}\mathbf{P}\mathbf{x},
\qquad
\dot V
=
-\mathbf{x}^{\top}\mathbf{Q}\mathbf{x}.
\end{equation}
Since $\mathbf{P}\succ0$ and $\mathbf{Q}\succ0$, the standard
quadratic-form bounds give
\begin{equation}
\lambda_{\min}(\mathbf{P})\|\mathbf{x}\|_2^2
\leq
V
\leq
\lambda_{\max}(\mathbf{P})\|\mathbf{x}\|_2^2,
\label{eq:V_bounds}
\end{equation}
where $\lambda_{\min}(\cdot)$ and $\lambda_{\max}(\cdot)$ denote the minimum and maximum eigenvalues of a symmetric matrix, respectively, and $\|\mathbf{x}\|_2$ denotes the Euclidean norm. Similarly,
\begin{equation}
\dot V
=
-\mathbf{x}^{\top}\mathbf{Q}\mathbf{x}
\leq
-\lambda_{\min}(\mathbf{Q})
\|\mathbf{x}\|_2^2.
\end{equation}
Using the upper bound in~\eqref{eq:V_bounds}, this becomes
\begin{equation}
\dot V
\leq
-\frac{\lambda_{\min}(\mathbf{Q})}
{\lambda_{\max}(\mathbf{P})}
V.
\label{eq:V_decay}
\end{equation}
Define
\begin{equation}
\gamma
=
\frac{\lambda_{\min}(\mathbf{Q})}
{\lambda_{\max}(\mathbf{P})}
>0.
\label{eq:gamma}
\end{equation}
Equation~\eqref{eq:V_decay} therefore implies
\begin{equation}
V(t)
\leq
V(0)e^{-\gamma t}.
\label{eq:V_exponential_decay}
\end{equation}
Combining~\eqref{eq:V_bounds} and
\eqref{eq:V_exponential_decay} yields
\begin{equation}
\|\mathbf{x}(t)\|_2
\leq
\sqrt{
\frac{\lambda_{\max}(\mathbf{P})}
{\lambda_{\min}(\mathbf{P})}
}
e^{-\gamma t/2}
\|\mathbf{x}(0)\|_2.
\label{eq:error_exponential_bound}
\end{equation}
Hence, $\mathbf{x}(t)$ converges exponentially to the origin, implying $e_i(t),\dot e_i(t)\to0$ for $i=1,2$, and establishing exponential stability of the joint-angle tracking-error subsystem.

\section{Numerical and Experimental Results}
\label{sec:results}
Since the proposed methodology comprises several stages, namely deformation estimation (Section~\ref{sec:deformation}), and its application in flexibility-aware kinematics (Section~\ref{sec:kinematics}), and control (Section~\ref{sec:control}), each step should be examined individually to ensure its overall effectiveness.

\subsection{Mode Convergence Analysis}
\label{subsec:convergence_results}

As discussed in Section~\ref{subsec:mode_shapes}, the primary parameters governing the flexibility of the $i^{\text{th}}$ link are the flexural rigidity $E_i I_i$, the mass per unit length $\rho_i A_i$, the link length $L_i$, and the payload mass $M_{p_i}$.

To evaluate the influence of these parameters, a flexible-link manipulator is simulated while enforcing the joint angle to follow a linear chirp trajectory defined as
\begin{equation}
\theta(t) = \theta_0  \sin \left[ \left( \frac{\Omega_{ch,\max}}{2T} \right) t^2 \right],
\end{equation}
where $\theta_0$ denotes the trajectory amplitude, $\Omega_{ch,\max}$ is the maximum instantaneous angular frequency, and $T$ represents the duration of the chirp signal. In this study, $\theta_0 = 10~\text{deg}$, $\Omega_{ch,\max} = \pi~\text{rad/s}$, and $T = 50~\text{s}$. This excitation provides a broad frequency sweep, ensuring sufficient stimulation of multiple vibration modes of the flexible link.

To investigate a wide range of flexible-link characteristics, simulations were performed over a broad parameter space considering two scenarios. All simulations were carried out in MATLAB/Simulink (R2024a) using a fixed-step solver with a sampling time of $0.001~\text{s}$ (1~kHz), ensuring accurate capture of the flexible dynamics. 

(i) In the first scenario, the effects of $E_i I_i$ and $\rho_i A_i$ were analyzed while keeping $L_i = 4.5~\text{m}$ and $M_{p_i} = 20~\mathrm{kg}$ fixed. (ii) In the second scenario, the effects of $L_i$ and $M_{p_i}$ were investigated while maintaining $E_i I_i = 74300~\text{Nm}^2$ and $\rho_i A_i = 5.37~\text{kg/m}$ constant. These parameter values were selected to match the characteristics of the experimental platform used in this study.

For both scenarios, the number of dominant modes of the $i^{\text{th}}$ link was determined using the criteria defined in Section~\ref{subsec:mode_selection}:

\subsubsection{Convergence of Observables Criterion Results}
\label{subsubsection:coo_results}

Using \eqref{eq:coo_mode_selection} with a threshold of $\delta = 10^{-6}$, the convergence of observables criterion was applied over the parameter ranges defined in the two scenarios. The resulting number of dominant modes is illustrated in Fig.~\ref{fig:mode_convergence1}.

\begin{figure}[ht]
    \centering
    \includegraphics[width=1\linewidth]{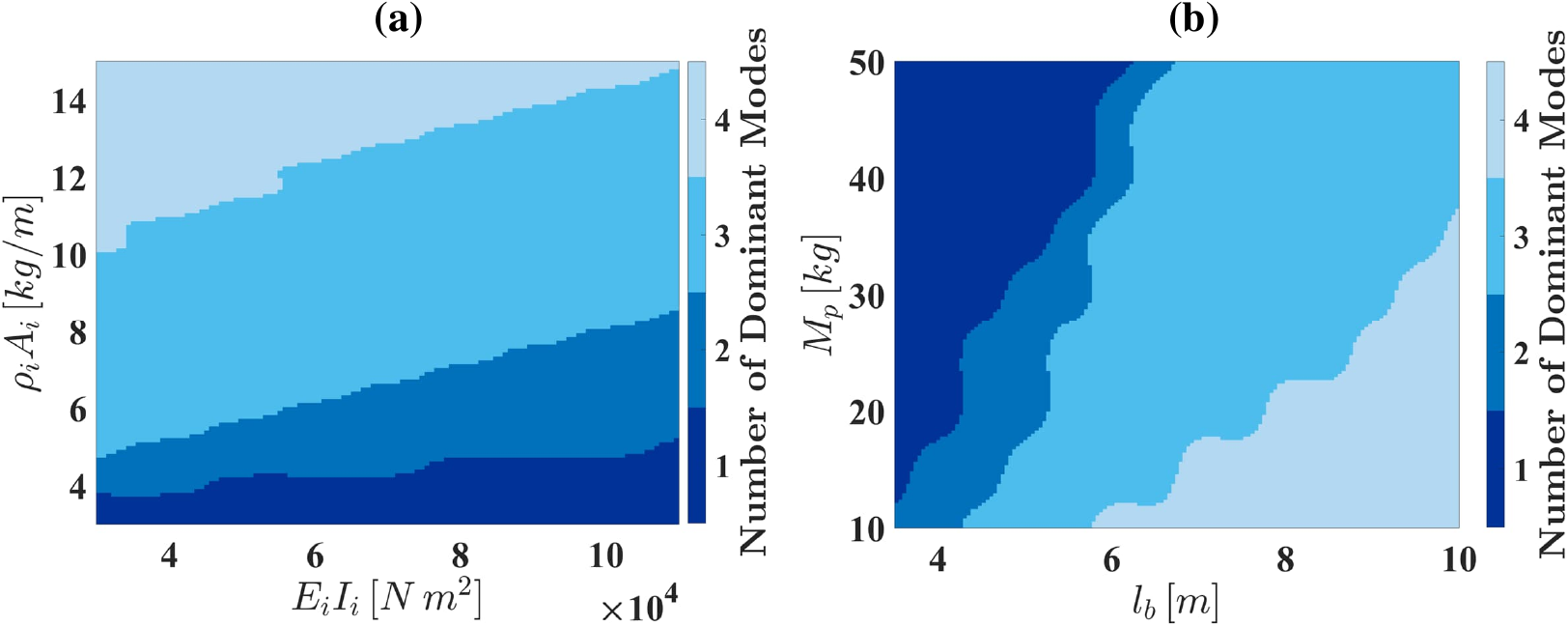}
    \caption{Number of dominant modes identified by the convergence-of-observables criterion: (a) effects of flexural rigidity $E_iI_i$ and mass per unit length $\rho_iA_i$; (b) effects of link length $L_i$ and equivalent tip mass $M_{p_i}$.}
    \label{fig:mode_convergence1}
\end{figure}

As shown in Fig.~\ref{fig:mode_convergence1}(a), links with lower flexural rigidity $E_i I_i$ and higher mass per unit length $\rho_i A_i$ require a larger number of dominant modes. This behavior arises because more flexible and heavier links exhibit richer vibration dynamics, causing higher-order modes to contribute more significantly to the overall response.

Fig.~\ref{fig:mode_convergence1}(b) further shows that longer links ($L_i$) are more sensitive to flexibility effects, which increases the number of dominant modes required for accurate modeling. In contrast, increasing the payload mass $M_{p_i}$ tends to reduce the number of dominant modes, as the added inertia suppresses higher-frequency vibrations of the flexible link.

\subsubsection{Natural Frequency Bandwidth Criterion Results}
\label{subsubsection:nfb_results}

Under the first scenario, the results shown in Fig.~\ref{fig:mode_convergence2} illustrate the influence of $E_i I_i$ and $\rho_i A_i$ on the first four natural frequencies.
The results indicate that links with lower $E_i I_i$ and higher $\rho_i A_i$ exhibit lower $f_{k_i}$; consequently, more mode shapes are excited according to the natural frequency bandwidth criterion. This observation is consistent with the conclusion obtained from the convergence of observables criterion.

\begin{figure}[ht]
    \centering
    \includegraphics[width=1\linewidth]{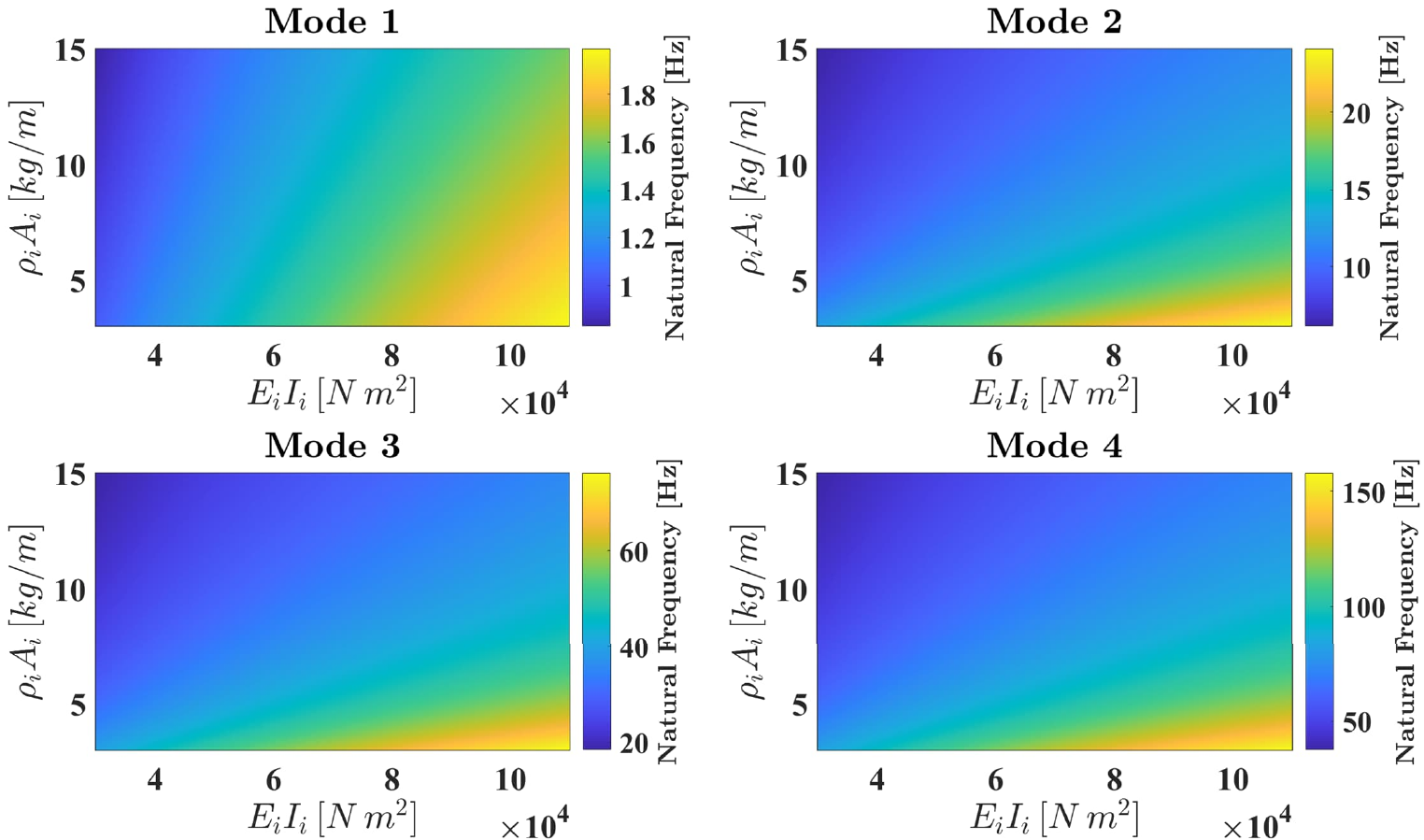}
    \caption{Variation of the first four natural frequencies with flexural rigidity $E_iI_i$ and mass per unit length $\rho_iA_i$ under the first scenario.}
    \label{fig:mode_convergence2}
\end{figure}

Fig.~\ref{fig:mode_convergence3} demonstrates the effects of $L_i$ and $M_{p_i}$ on the natural frequencies under the second scenario. 
It is observed that links with larger $L_i$ exhibit lower natural frequencies and are therefore more susceptible to excitation according to the natural frequency bandwidth criterion. 
This observation aligns with the conclusions derived from the convergence of observables criterion.

\begin{figure}[h]
    \centering
    \includegraphics[width=1\linewidth]{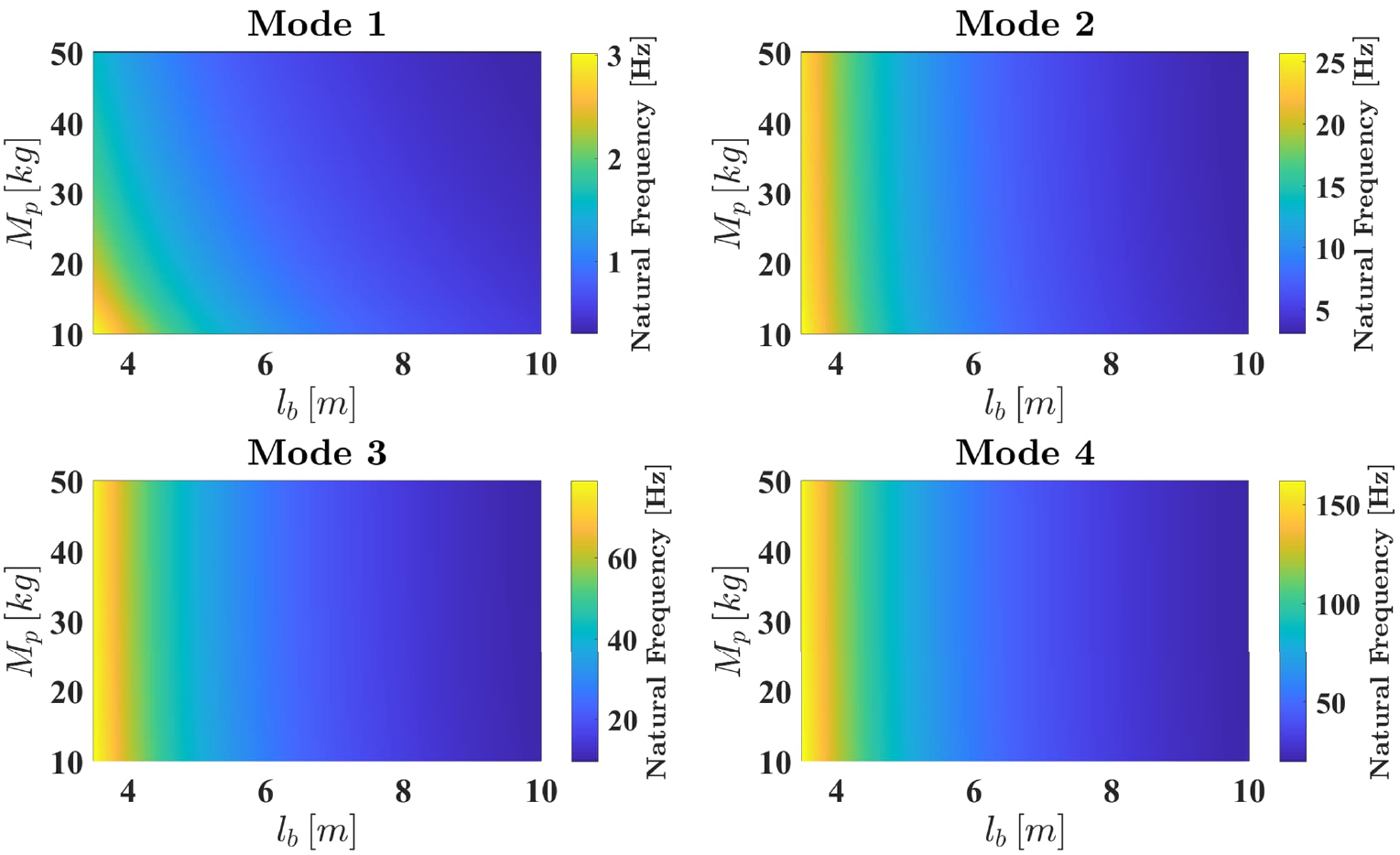}
    \caption{Variation of the first four natural frequencies with link length $L_i$ and equivalent tip mass $M_{p_i}$ under the second scenario.}
    \label{fig:mode_convergence3}
\end{figure}

\subsection{Experimental Results of Deformation Estimation}
\label{subsec:deformation_results}

Having examined the influence of different parameters on the number of dominant mode shapes, the proposed deformation estimation method introduced in Section~\ref{subsec:IMU_network} is next validated on a real flexible robotic system. Before extending the study to serial flexible manipulators, it is essential to evaluate the proposed deformation estimation method in isolation, therefore, a SLFM is employed as an experimental platform to rigorously validate the deformation estimation approach. The experimental setup is shown in Fig.~\ref{fig:SLFM_Setup}.

The flexible link is a 4.5~m-long beam fabricated from high-strength structural steel (OPTIM 700 MH Plus) with a square hollow cross-section of 60~mm $\times$ 60~mm and a wall thickness of 3~mm. Motion of the link is generated by a $\oslash$35/25--300~mm hydraulic cylinder driven by a Bosch Rexroth 4WRPEH high-response valve with a nominal bandwidth of 100~Hz. To monitor the response of the system, multiple Murata SCHA63T-K03 IMUs are distributed along the link, enabling real-time kinematic measurements. For experimental validation, the end-effector pose is measured using a Leica Absolute Laser Tracker AT960-LR, which serves as the ground-truth reference system.

All experiments were conducted on a Speedgoat real-time target machine with a sampling time of $0.001~\mathrm{s}$ (1~kHz). All system sensors were interfaced with the Speedgoat target through EtherCAT communication. 
The filter parameters were set to $k_P = 7.0407$ and $k_I = 0.1984$, and kept constant across all experiments to ensure consistent evaluation conditions~\cite{barjini2025ai}.

\begin{figure}[ht]
    \centering
    \includegraphics[width=0.6\linewidth]{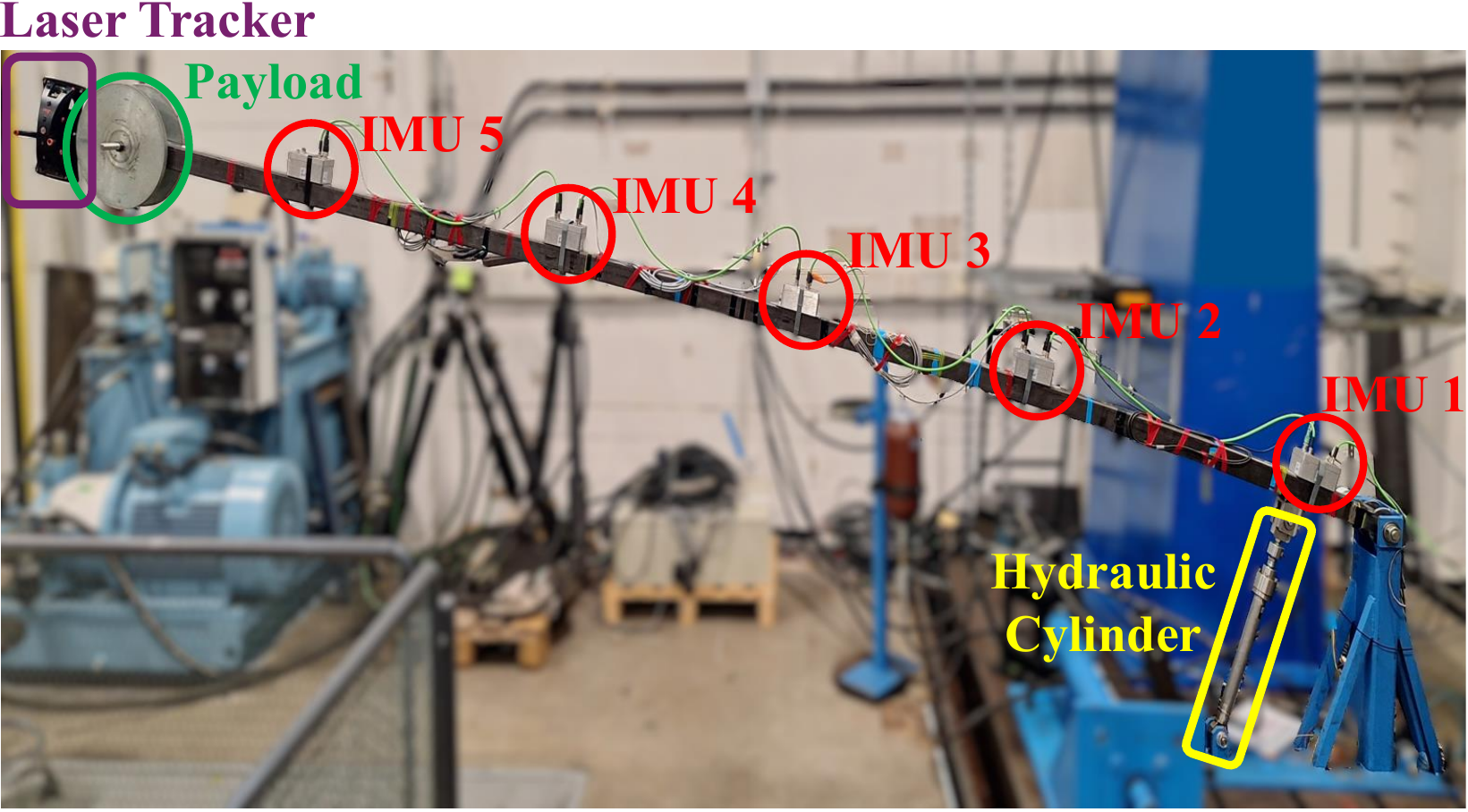}
    \caption{Experimental SLFM platform for deformation-estimation validation using distributed IMUs and laser-tracker ground truth.}
    \label{fig:SLFM_Setup}
\end{figure}

\begin{remark}
The SLFM, shown in Fig.~\ref{fig:SLFM_Setup}, has fixed material and geometric properties, characterized by a constant flexural rigidity $E I = 74300~\text{Nm}^2$, mass per unit length $\rho A = 5.37~\text{kg/m}$, and link length $L = 4.5~\text{m}$. 
In practical operation, however, the manipulator is required to handle varying payloads ($M_{p}$). Therefore, the effect of payload variation on the number of dominant modes is of particular interest. Based on the mode convergence analysis presented in Section~\ref{subsec:convergence_results} (see Fig.~\ref{fig:mode_convergence1}(b)), the system exhibits two dominant mode shapes for $M_{p} = 20~\mathrm{kg}$ and one dominant mode shape for $M_{p} = 40~\mathrm{kg}$, in simulation. 
These findings motivate further experimental validation to assess the impact of payload variations on the deformation behavior of the flexible link.
\label{remark:deformation_results}
\end{remark}

The performance of the proposed deformation estimation method is evaluated against the following approaches:
\begin{itemize}
    \item \textit{Piecewise Rigid:} A rigid multi-segment approximation consisting of five rigidly connected segments, implemented using five IMUs as described in~\cite{barjini2025ai}.
    \item \textit{1 Mode:} The proposed formulation with a single dominant mode shape, utilizing IMUs 1 and 5.
    \item \textit{2 Modes:} The proposed formulation with two dominant mode shapes, utilizing IMUs 1, 3, and 5.
    \item \textit{3 Modes:} The proposed formulation with three dominant mode shapes, utilizing IMUs 1, 2, 3, and 5.
\end{itemize}
This comparison enables a systematic investigation of the effect of increasing modal complexity on deformation estimation accuracy and overall performance.

\begin{figure}[htp]
    \centering
    \includegraphics[width=0.95\linewidth]{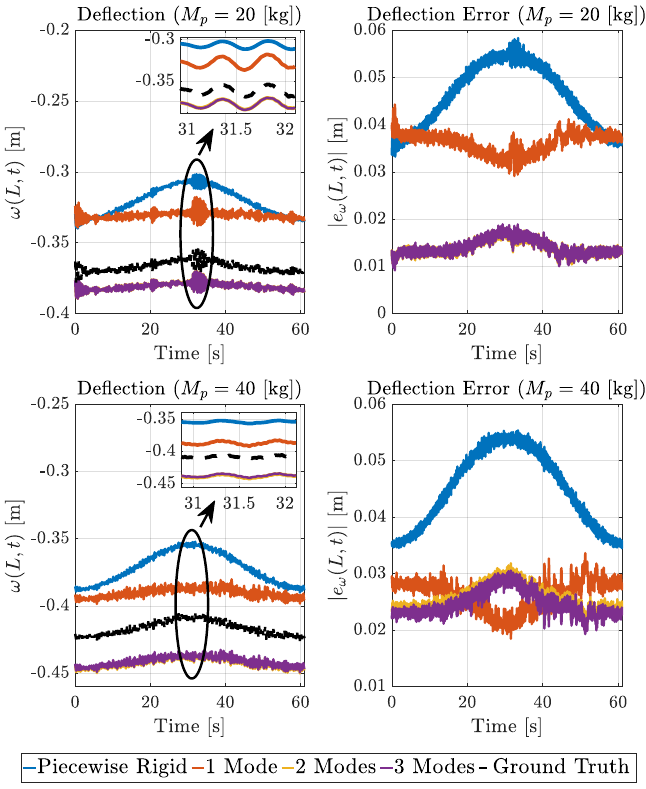}
    \caption{Experimental results of deformation estimation and corresponding estimation errors for the SLFM under 20\,[kg] and 40\,[kg] payloads.}
    \label{fig:deformation_SLFM}
\end{figure}

To experimentally validate the proposed deformation estimation framework, the SLFM was commanded to track the sinusoidal joint trajectory
\begin{equation}
    \theta_d(t)=\bar{\theta}-A\cos(\Omega_{tr} t),
    \label{eq:joint_trajectory}
\end{equation}
where $\bar{\theta}=25^\circ$ ($0.436~\mathrm{rad}$), $A=5^\circ$ ($0.087~\mathrm{rad}$), and $\Omega_{tr}=0.1~\mathrm{rad/s}$, resulting in a joint motion between $20^\circ$ and $30^\circ$. To evaluate the effect of payload variations, two payload configurations, $M_p=20~\mathrm{kg}$ and $M_p=40~\mathrm{kg}$, were considered in accordance with Remark~\ref{remark:deformation_results}. The corresponding deformation responses are presented in Fig.~\ref{fig:deformation_SLFM}.

The quantitative deformation-estimation results are summarized in Table~\ref{tab:deformation_rmse}. Consistent with Fig.~\ref{fig:deformation_SLFM}, the piecewise-rigid approximation exhibits the largest RMSE for both payload conditions despite using the largest number of IMUs, highlighting the importance of incorporating the flexible-link mode shapes into the deformation reconstruction.
For $M_p=20~\mathrm{kg}$, the 2 Modes formulation achieves the lowest RMSE of $0.014001~\mathrm{m}$, while including a third mode provides no further improvement. In contrast, for $M_p=40~\mathrm{kg}$, the 1 Mode formulation already reduces the RMSE to $0.026317~\mathrm{m}$, and the inclusion of additional modes yields only marginal improvement. These results are in agreement with Remark~\ref{remark:deformation_results}, indicating that higher-order modes become less influential as the payload increases, whereas the lighter-payload case requires additional modal content for accurate deformation reconstruction.

\begin{table}[htp]
\caption{Experimental tip-deformation RMSE for different deformation reconstruction methods.}
\label{tab:deformation_rmse}
\centering
\begin{tabular}{l|c|c}
\hline
Method & $M_p=20$ kg [m] & $M_p=40$ kg [m] \\
\hline
Piecewise Rigid & 0.046593 & 0.045521 \\
1 Mode          & 0.036284 & 0.026317 \\
2 Modes         & 0.014001 & 0.025744 \\
3 Modes         & 0.014123 & 0.024817 \\
\hline
\end{tabular}
\end{table}

\subsection{Experimental Results of Flexibility-Aware Kinematics}
\label{subsec:kinematics_results}

After validating the proposed deformation estimation method in isolation using the SLFM, its integration into the flexibility-aware forward kinematics of Section~\ref{sec:kinematics} is evaluated on a serial flexible manipulator. This experiment verifies that the estimated link deformations can be effectively incorporated into the kinematic model to improve end-effector pose estimation.

The experimental platform is a DLFM, representing a serial flexible robotic system, as shown in Fig.~\ref{fig:kinematics_Setup}. Both links have the same profile as the SLFM in Section~\ref{subsec:deformation_results}, including the flexural rigidity $EI$, and mass per unit length $\rho A$. The link lengths are $L_1=4.5~\text{m}$ and $L_2=3.55~\text{m}$. As in the SLFM experiments, a Leica Absolute Laser Tracker AT960-LR provides the ground-truth end-effector kinematics. Each link is also equipped with the same IMU network, enabling joint angle estimation and link deformation reconstruction. The joints are additionally equipped with SICK AFM60A absolute encoders with an 18-bit single-turn resolution, which are used only for validation of the IMU-based joint-angle estimates. The same Speedgoat-based EtherCAT acquisition architecture described in Section~\ref{subsec:deformation_results} was used for the DLFM experiments.

\begin{figure}[ht]
    \centering
    \includegraphics[width=1\linewidth]{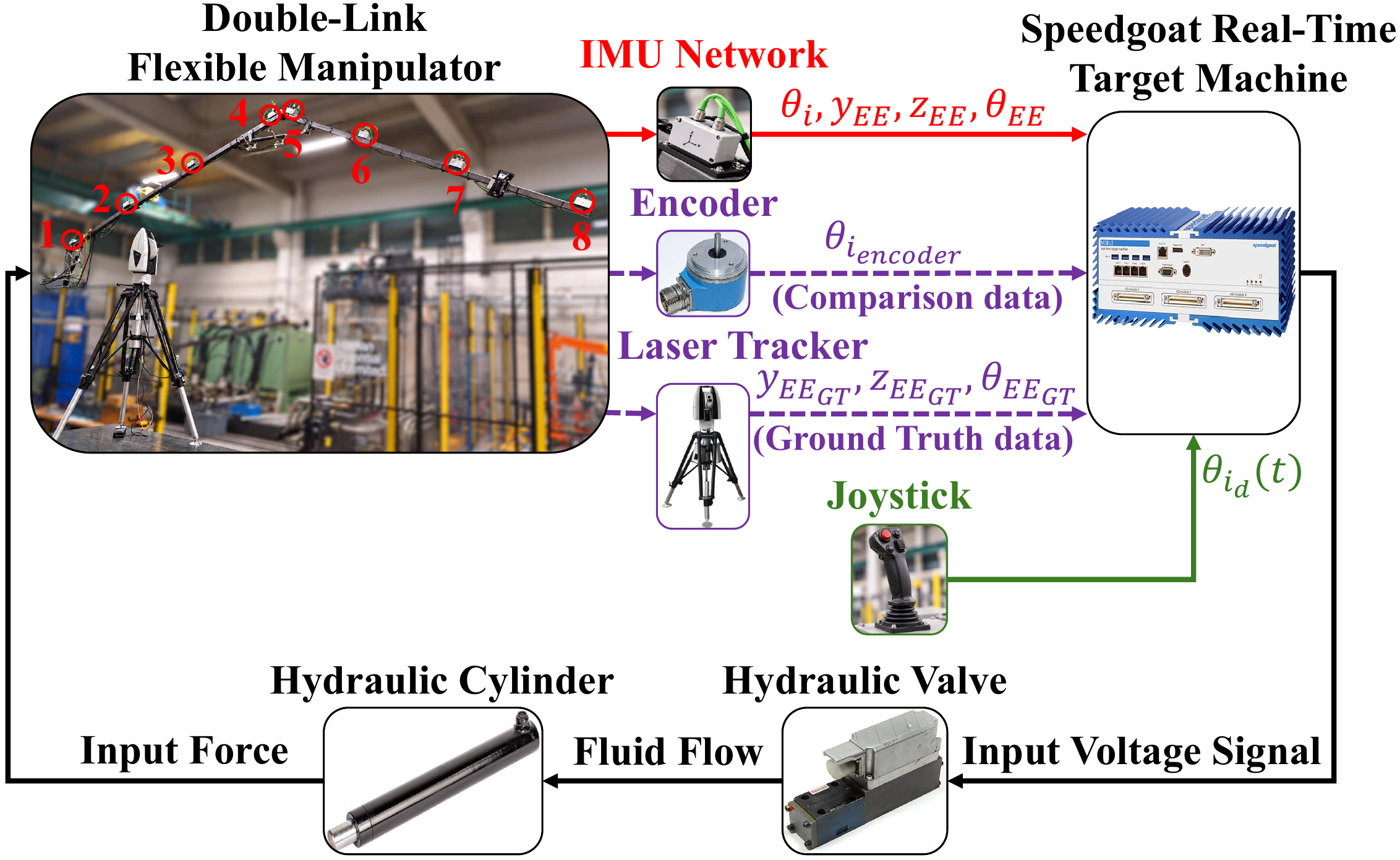}
    \caption{Experimental DLFM platform for validation of the proposed flexibility-aware forward kinematics.}
    \label{fig:kinematics_Setup}
\end{figure}

For joint angle estimation, $\theta_1$ is obtained using IMU~1 mounted on the first link, while $\theta_2$ is computed from the relative orientation between the last IMU of the first link and the first IMU of the second link, following Section~\ref{sec:preliminaries}. The estimated joint angles are compared with encoder measurements in Fig.~\ref{fig:theta_FK_DLFM}, where their close agreement confirms the reliability of the joint angle estimation method, presented in Section~\ref{sec:preliminaries}.
\begin{figure}[h]
    \centering
    \includegraphics[width=0.9\linewidth]{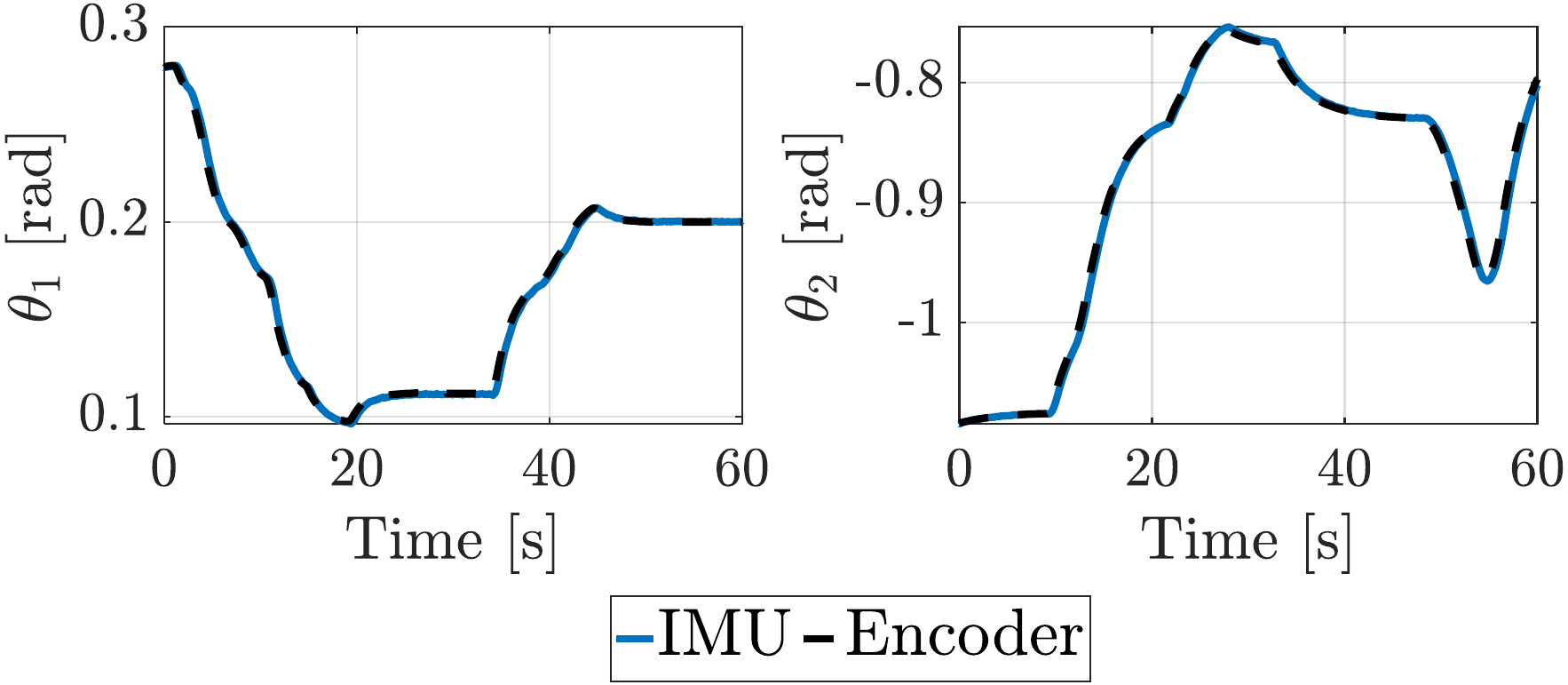}
    \caption{Experimental comparison of IMU-estimated and encoder-measured joint angles $\theta_1$ and $\theta_2$ of the DLFM.}
    \label{fig:theta_FK_DLFM}
\end{figure}

\begin{remark}
For the DLFM experiment, the number of dominant modes for each link is selected according to the mode convergence analysis in Section~\ref{subsec:convergence_results}. As shown in Fig.~\ref{fig:mode_convergence1}(b), the first link, with $L_1=4.5~\text{m}$ and carrying the second link, its actuator, and the payload with a total mass of approximately $30~\mathrm{kg}$, exhibits one dominant mode shape. In contrast, the second link, with $L_2=3.55~\text{m}$ and carrying an approximately $10~\text{kg}$ payload, exhibits two dominant mode shapes. Therefore, one mode is used for the first link and two modes are used for the second link in the proposed flexibility-aware forward kinematics implementation.
\label{remark:kinematics_results}
\end{remark}

The performance of the flexibility-aware forward kinematics is evaluated against the following approaches:
\begin{itemize}
    \item \textit{Rigid:} The links are assumed to be rigid, and no deformation is considered. The end-effector position is computed using only the IMU-estimated joint angles.
    
    \item \textit{Piecewise Rigid:} Each link is approximated by four rigidly connected segments using all eight IMUs mounted on the DLFM, as described in~\cite{barjini2025ai}.
    
   \item \textit{Proposed:} The proposed deformation estimation method is integrated with the flexibility-aware forward kinematics in Section~\ref{sec:kinematics}. Based on Remark~\ref{remark:kinematics_results}, IMUs~1 and~4 are used for the first link to capture its single dominant mode, while IMUs~5, 6, and~8 are used for the second link to capture its two dominant modes.
\end{itemize}

In this experiment, the DLFM joint angles were manually commanded using a joystick, as illustrated in Fig.~\ref{fig:theta_FK_DLFM}, and random configurations were selected to evaluate the forward kinematics over different postures. Fig.~\ref{fig:FK_DLFM} shows the estimated end-effector pose, including $y_{EE}$, $z_{EE}$, and $\theta_{EE}$, together with the corresponding absolute errors. The rigid model exhibits the largest errors because it neglects link deformation. The piecewise-rigid approach improves the estimates using all eight IMUs, but its accuracy remains limited by the rigid-segment approximation. In contrast, the proposed flexibility-aware forward kinematics achieves the lowest position and orientation errors by exploiting the dominant mode shapes.

\begin{figure}[htp]
    \centering
    \includegraphics[width=0.95\linewidth]{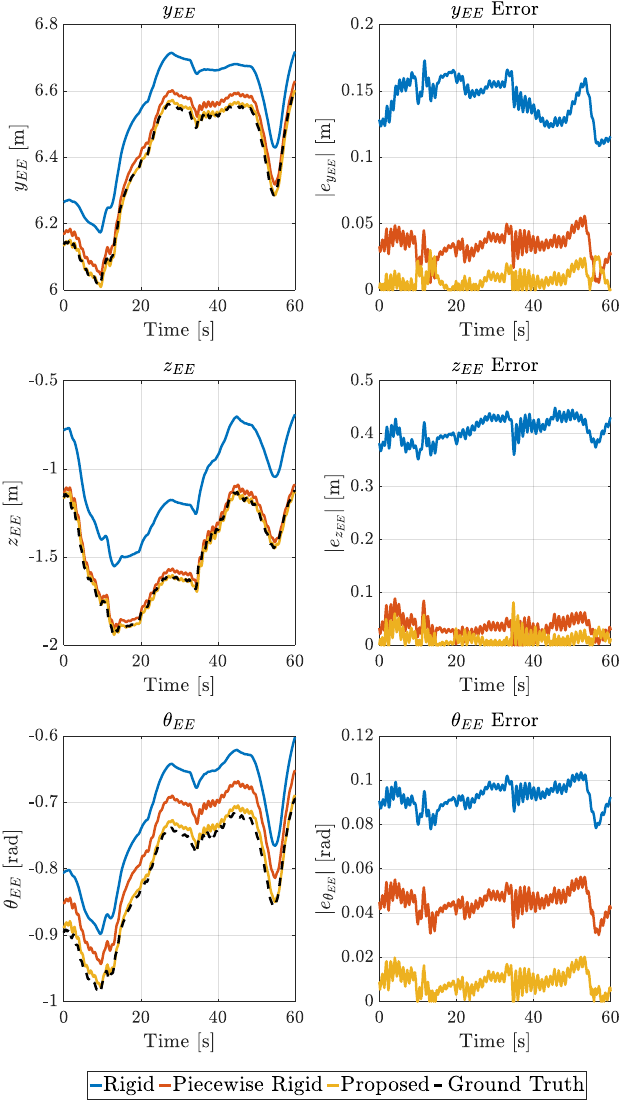}
    \caption{Experimental DLFM end-effector pose estimation and corresponding absolute errors for the rigid, piecewise-rigid, and proposed flexibility-aware models.}
    \label{fig:FK_DLFM}
\end{figure}

As summarized in Table~\ref{tab:FK_DLFM_RMSE}, the proposed method reduces the RMSEs to $0.0107~\mathrm{m}$, $0.0194~\mathrm{m}$, and $0.0102~\mathrm{rad}$ for $y_{EE}$, $z_{EE}$, and $\theta_{EE}$, respectively. Moreover, it uses only five IMUs, compared with eight for the piecewise-rigid approach, demonstrating improved accuracy and sensor efficiency.
\begin{table}[htp]
\caption{RMSE of end-effector pose estimation.}
\label{tab:FK_DLFM_RMSE}
\centering
\begin{tabular}{l|c|c|c}
\hline
Method & $y_{EE}$ RMSE & $z_{EE}$ RMSE & $\theta_{EE}$ RMSE \\
       & [m]      & [m]      & [rad] \\
\hline
Rigid           & 0.144911 & 0.407625 & 0.092640 \\
Piecewise rigid & 0.035099 & 0.037915 & 0.045337 \\
Proposed        & \textbf{0.010725} & \textbf{0.019378} & \textbf{0.010170} \\
\hline
\end{tabular}
\end{table}

\subsection{Experimental Results of Flexibility-Aware Control}
\label{subsec:control_results}

After validating the proposed deformation estimation method and its integration into flexibility-aware forward kinematics, the reconstructed deformation is incorporated into the closed-loop controller developed in Section~\ref{sec:control}. The objective of this experiment is to demonstrate the real-time use of IMU-estimated joint states and reconstructed flexible-link deformation within the proposed deformation-informed closed-loop controller on the DLFM.

The controller was implemented on the same experimental platform described in Subsection~\ref{subsec:kinematics_results}, using the IMU configuration specified in Remark~\ref{remark:kinematics_results}, as illustrated in Fig.~\ref{fig:control_Setup}. For the closed-loop control experiments, the hydraulic-cylinder chamber pressures were measured using Druck UNIK 5000 pressure transmitters with an operating pressure range of 25~MPa.

\begin{figure}[ht]
    \centering
    \includegraphics[width=0.72\linewidth]{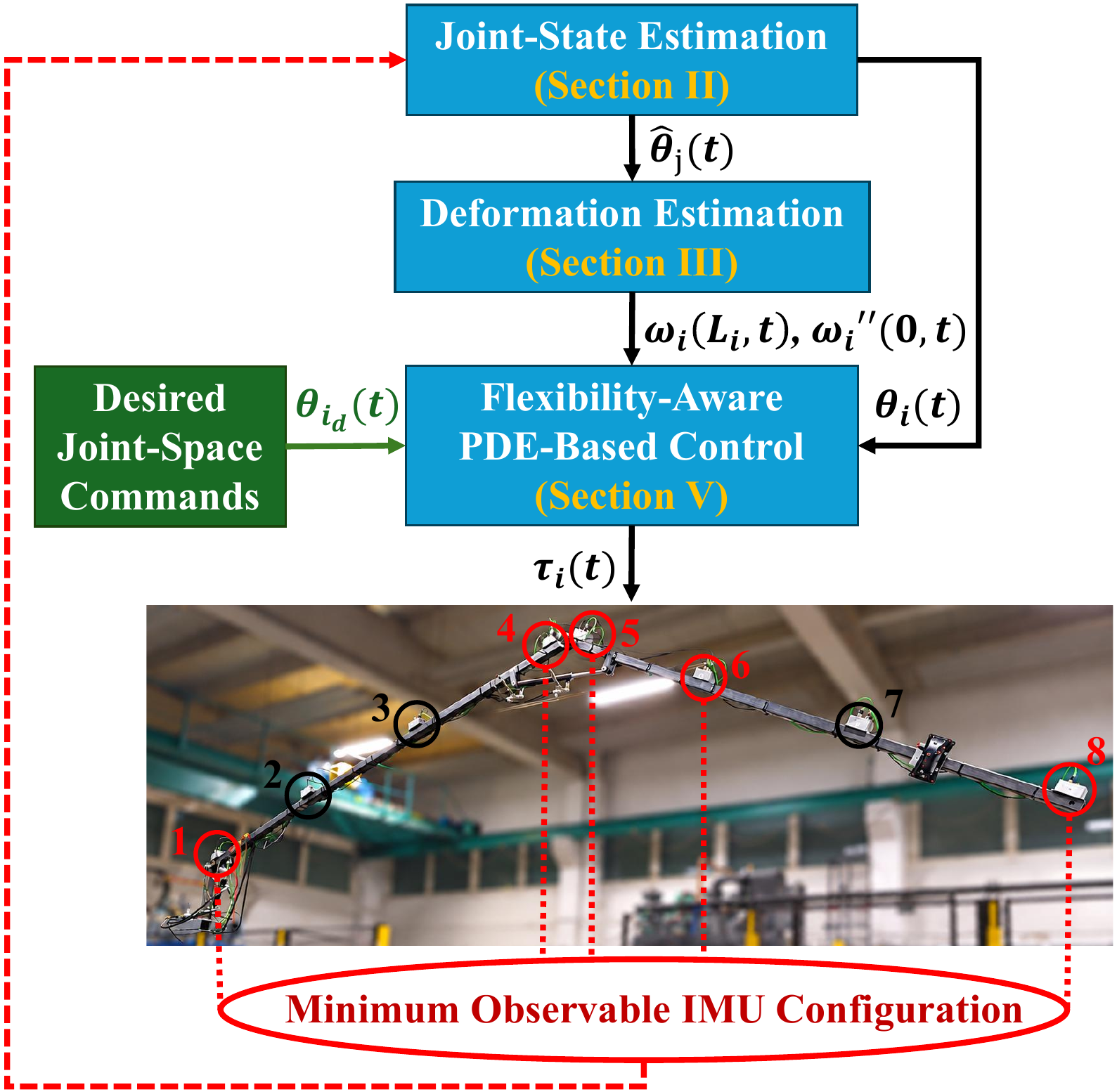}
    \caption{Overview of the proposed flexibility-aware control.}
    \label{fig:control_Setup}
\end{figure}

During operation, the proposed IMU-based framework provides the joint angles $\theta_1(t)$ and $\theta_2(t)$, the reconstructed tip deformations $\omega_1(L_1,t)$ and $\omega_2(L_2,t)$, and the corresponding base curvatures $\omega_1''(0,t)$ and $\omega_2''(0,t)$ required by the control law~\eqref{eq:tau_i}. Thus, joint-state feedback is obtained exclusively from the IMU measurements, without using encoder measurements in the closed-loop control algorithm.

The controller gains were selected as $K_{p_1}=6000$, $K_{d_1}=6000$, $K_{p_2}=12000$, and $K_{d_2}=2000$. The proposed controller computes the joint torque commands according to~\eqref{eq:tau_i}, which are subsequently applied to the DLFM through the hydraulic actuation system.

To evaluate the closed-loop tracking performance, a point-to-point reaching task was considered. The desired joint trajectories were generated using a quintic polynomial time-scaling law, ensuring continuous position, velocity, and acceleration, with zero velocity and acceleration at the beginning and end of each motion. For the $i^\text{th}$ joint, the desired trajectory over the interval $t\in[t_0,t_f]$ is defined as
\begin{equation}
\theta_{d_i}(t)
=
\theta_{i,0}
+
\left(\theta_{i,f}-\theta_{i,0}\right)
\left(
10\sigma^3-15\sigma^4+6\sigma^5
\right),
\label{eq:quintic_trajectory}
\end{equation}
where
\begin{equation}
\sigma=\frac{t-t_0}{t_f-t_0}, \qquad 0\leq\sigma\leq1,
\end{equation}
and $\theta_{i,0}$ and $\theta_{i,f}$ denote the initial and final joint configurations, respectively. The trajectory satisfies the boundary conditions
\begin{equation}
\begin{aligned}
\theta_{d_i}(t_0) &= \theta_{i,0}, &
\theta_{d_i}(t_f) &= \theta_{i,f},\\
\dot{\theta}_{d_i}(t_0) &= \dot{\theta}_{d_i}(t_f)=0, &
\ddot{\theta}_{d_i}(t_0) &= \ddot{\theta}_{d_i}(t_f)=0.
\end{aligned}
\end{equation}
The experimental results for a representative reaching trial are presented in Figs.~\ref{fig:theta_control}--\ref{fig:omega_control}.

Fig.~\ref{fig:theta_control} presents the desired joint trajectories, the corresponding IMU-estimated joint angles, and the resulting tracking errors. The estimated joint states closely follow the commanded quintic trajectories throughout the representative reaching task, with the tracking errors remaining bounded during motion and decreasing once the target configurations are reached. These results demonstrate the real-time integration of IMU-estimated joint states and reconstructed link deformation within the deformation-informed closed-loop PDE controller.

\begin{figure}[htp]
    \centering
    \includegraphics[width=0.9\linewidth]{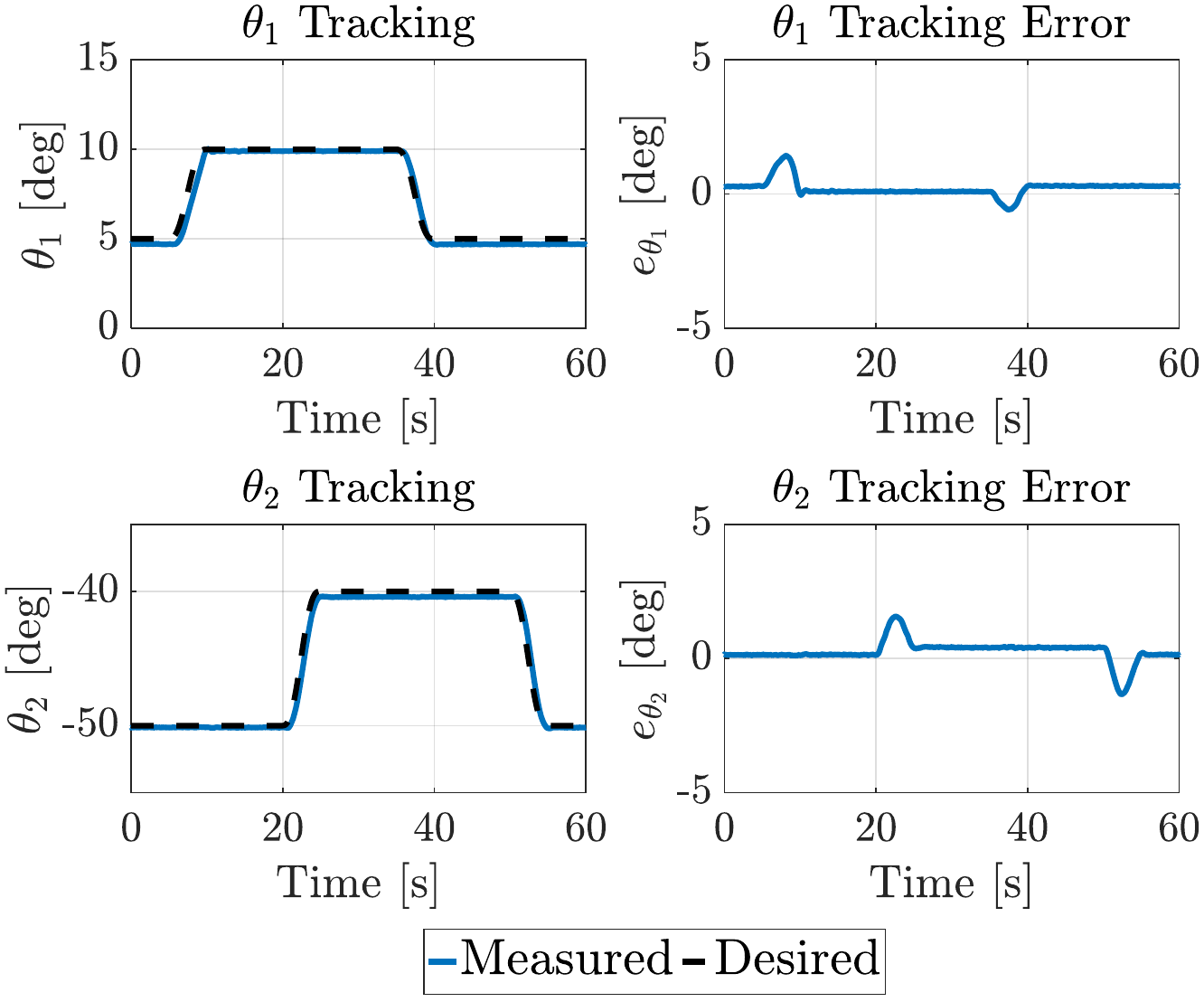}
    \caption{Experimental joint-space tracking performance of the deformation-informed PDE controller during the representative reaching task.}
    \label{fig:theta_control}
\end{figure}

The corresponding control effort is shown in Fig.~\ref{fig:tau_control}. 
The hydraulic-cylinder forces were estimated from the measured chamber pressures and the corresponding effective piston areas, and were subsequently converted into equivalent joint torques using the configuration-dependent actuator moment arms. 
The observed torque variations reflect the combined effects of joint motion, gravitational loading, and deformation-dependent dynamics during the reaching task. The experimental tracking and control-effort performance is summarized in Table~\ref{tab:control_performance}.

\begin{table}[htp]
\caption{Experimental joint tracking and control-effort performance.}
\label{tab:control_performance}
\centering
\begin{tabular}{l|c|c}
\hline
Joint & Tracking RMSE [deg] & Torque RMS [N\,m] \\
\hline
Joint 1 & 0.363 & 3179.57 \\
Joint 2 & 0.477 & 310.40 \\
\hline
\end{tabular}
\end{table}

\begin{figure}[htp]
    \centering
    \includegraphics[width=1\linewidth]{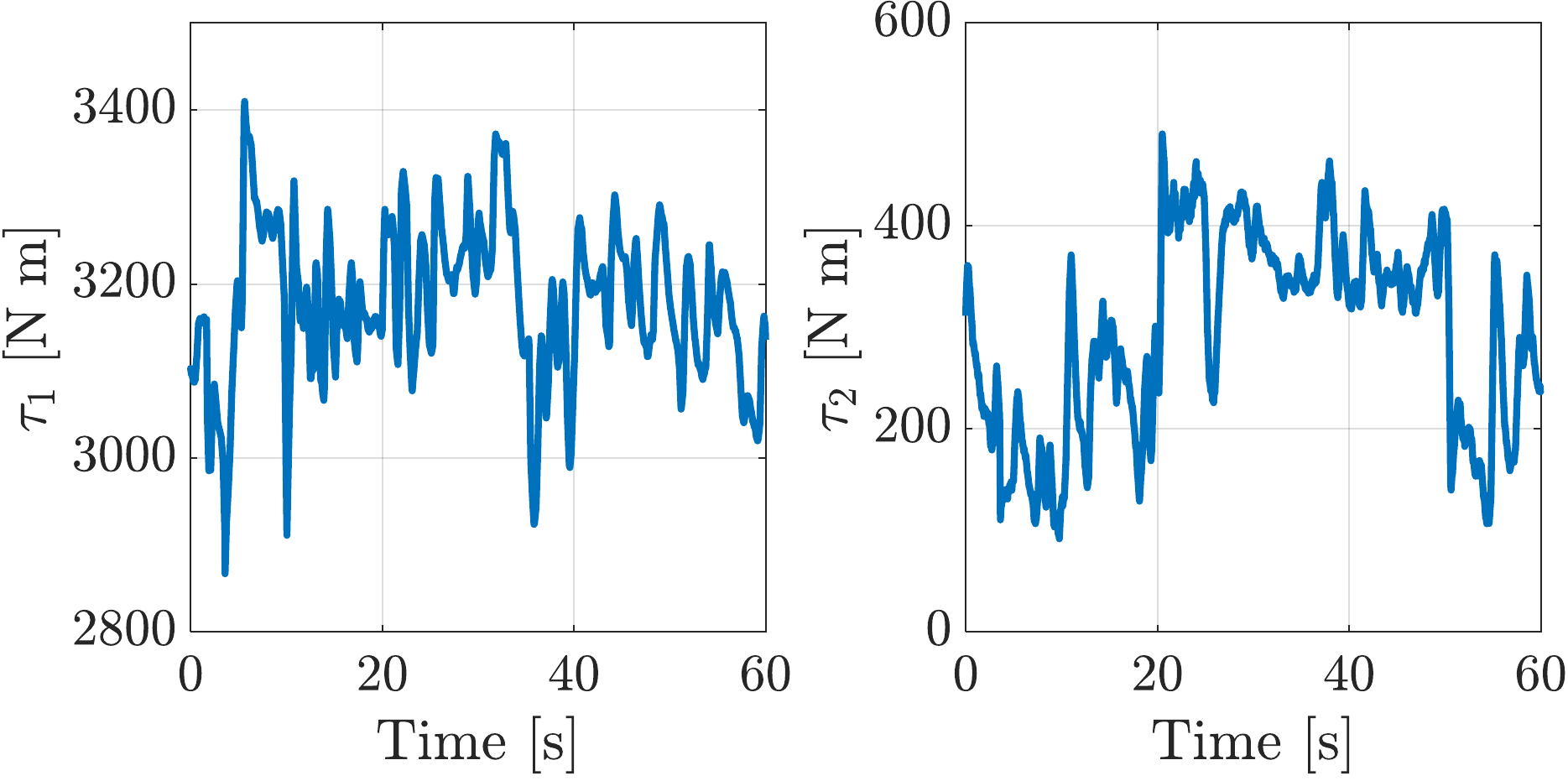}
    \caption{Experimental control effort of the proposed deformation-informed PDE controller during the reaching task.}
    \label{fig:tau_control}
\end{figure}

Fig.~\ref{fig:omega_control} presents the reconstructed tip deformations $\omega_1(L_1,t)$ and $\omega_2(L_2,t)$ during the same experiment. The deformation evolves continuously with the manipulator configuration and motion, while the oscillatory variations remain small throughout the reaching task. These results confirm that the flexible-link states required by the controller remain available online during closed-loop operation and demonstrate the real-time integration of the proposed deformation reconstruction method with the PDE-based control framework.

\begin{figure}[htp]
    \centering
    \includegraphics[width=1\linewidth]{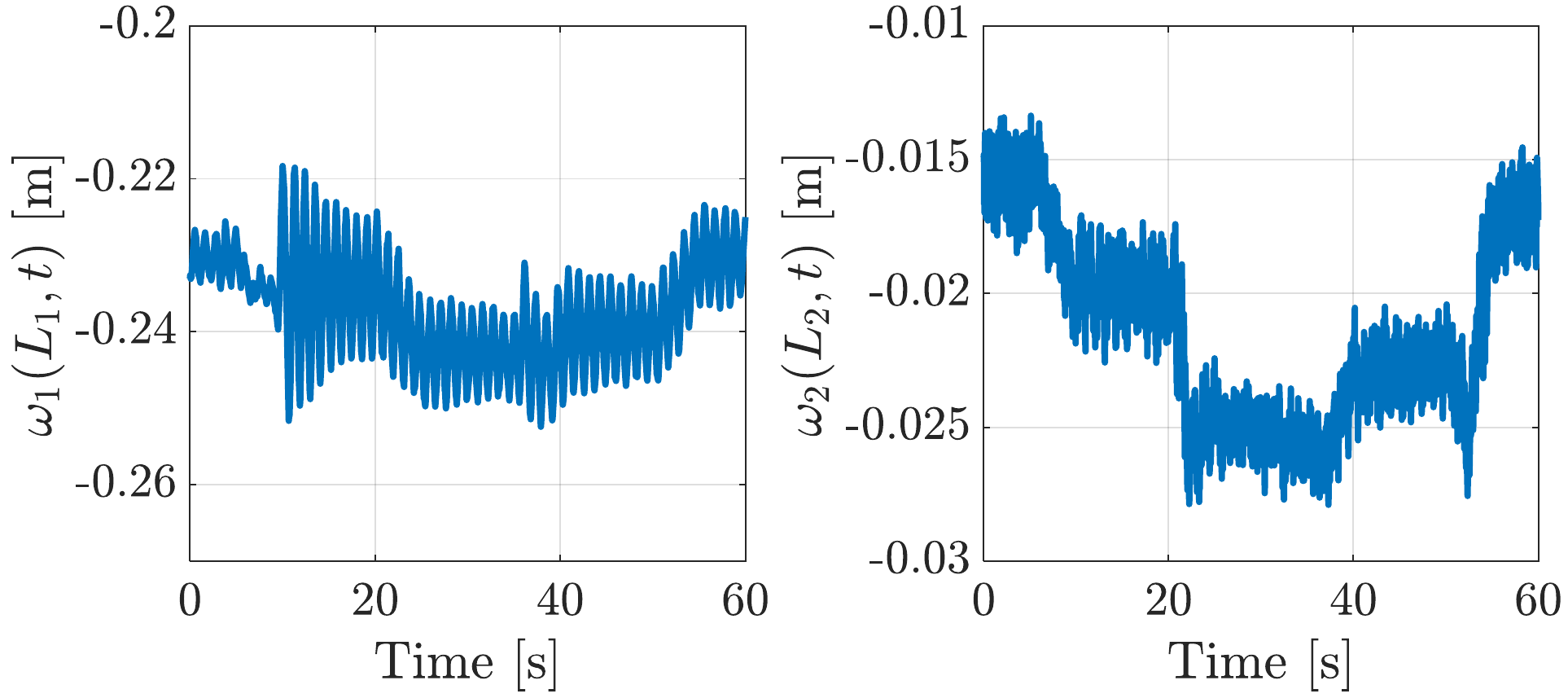}
    \caption{IMU-reconstructed flexible-link tip deformations during the representative closed-loop reaching task.}
    \label{fig:omega_control}
\end{figure}

Overall, the experimental results demonstrate the feasibility of incorporating deformation information reconstructed from the distributed IMU network into the proposed PDE-based joint-space controller in real time while maintaining accurate tracking during the considered reaching task. 
Together with the preceding deformation-estimation and kinematic experiments, these results establish the complete sensing-to-control framework from distributed IMU measurements to deformation reconstruction, flexibility-aware kinematics, and real-time closed-loop control.

\section{Conclusion}
\label{sec:conclusion}

This paper presented a unified flexibility-aware manipulation framework that connects the structural dynamics of flexible links to sensing, deformation reconstruction, kinematics, and closed-loop control. The key idea is to determine the sensing requirement from the dominant modal content of each link rather than relying on a predefined dense IMU network. 

Based on response-convergence and excitation-bandwidth criteria, the dominant modes are identified, and the corresponding modal observability condition defines the minimum observable IMU configuration. The distributed inertial measurements are then used to reconstruct the continuous link deformation and incorporate it into flexibility-aware forward kinematics and a deformation-informed PDE-based controller.

Numerical and experimental results on single- and double-link hydraulic flexible manipulators demonstrate the effectiveness of the framework. For the SLFM, the mode-informed reconstruction substantially reduced the tip-deformation error compared with the piecewise-rigid approximation. On the long-reach DLFM, the proposed kinematics achieved end-effector RMSEs of $0.0107$~m, $0.0194$~m, and $0.0102$~rad in $y_{EE}$, $z_{EE}$, and $\theta_{EE}$, respectively, while outperforming the piecewise-rigid method with fewer sensors. The same reconstructed deformation states were used online in the PDE-based controller, achieving joint-tracking RMSEs of $0.363^\circ$ and $0.477^\circ$ in the representative reaching experiment. Together, these results demonstrate a complete sensing-to-control pipeline in which distributed IMU measurements support deformation estimation, end-effector pose estimation, and real-time feedback control.

Future work will extend the framework to spatial flexible manipulators, larger deformation regimes, and online adaptation of modal content under changing payloads and operating conditions.

\bibliographystyle{IEEEtran}
\bibliography{MyReferences.bib}

\end{document}